\documentclass{article}

\usepackage[final]{nips_2017}

\usepackage[utf8]{inputenc} 
\usepackage[T1]{fontenc}

\PassOptionsToPackage{hyphens}{url}\usepackage{hyperref}
\usepackage{booktabs}       
\usepackage{amsfonts}       
\usepackage{nicefrac}       
\usepackage{microtype}      
\usepackage{wrapfig,lipsum,booktabs}
\usepackage{multicol}
\usepackage{enumitem}

\usepackage{graphicx} 

\usepackage{subfigure}

\usepackage{algorithm}
\usepackage[noend]{algpseudocode}

\usepackage{bm,color}
\usepackage{mathdots}
\usepackage{float}
\usepackage{amsmath, amssymb,amsthm}
\usepackage{xspace}
\usepackage{array}

\newcolumntype{C}[1]{>{\centering\arraybackslash}m{#1}}
\usepackage{tabu}
\usepackage{subfigure}
\usepackage{enumitem}
\graphicspath{ {./}{./plots/} }
\usepackage{capt-of}
\usepackage{caption}

\usepackage{capt-of}
\usepackage{ mathrsfs }
\usepackage{multirow}
\graphicspath{ {./}{./plots/} }

\newcommand{\cocoa}{\textsc{CoCoA}\xspace}
\newcommand{\cocoamtl}{\textsc{Mocha}\xspace}

\newcommand{\eqdef}{:=}
\newcommand{\R}{\mathbb{R}}                      
\newcommand{\Prob}{\mathbb{P}}                   
\newcommand{\E}{\mathbb{E}}                      
\newcommand{\Exp}{\mathbb{E}}

\newcommand{\xv}{ {\bf x}}
\newcommand{\yv}{ {\bf y}}

\newcommand{\uv}{ {\bf u}}
\newcommand{\vv}{ {\bf v}}
\newcommand{\wv}{ {\bf w}}
\newcommand{\alphav}{ {\boldsymbol \alpha}}

\newcommand{\X}{ {\bf X}}

\newcommand{\0}{ {\bf 0}}
\newcommand{\Wv}{ {\bf W}}
\newcommand{\Ov}{ {\bf \Omega}}
\newcommand{\Xv} {{\bf X}}

\newcommand{\1} {{\bf 1}}
\newcommand{\Iv}{{\bf I}}

\newcommand{\Mv}{{\bf M}}

\newcommand{\aggpar}{\gamma}

\newcommand{\vc}[2]{#1^{(#2)}}

\newcommand{\trace}[1]{\text{tr}{\left( {#1} \right)}}

\newcommand{\bP}{\mathcal{P}}
\newcommand{\bD}{\mathcal{D}}
\newcommand{\bH}{\mathcal{H}}

\newcommand{\bR}{\mathcal{R}}

\newcommand{\Ggk}{\mathcal{G}^{\sigma'}_t\hspace{-0.08em}}

\newtheorem*{rep@theorem}{\rep@title}
\newcommand{\newreptheorem}[2]{
\newenvironment{rep#1}[1]{
 \def\rep@title{#2 \ref{##1}}
 \begin{rep@theorem}}
 {\end{rep@theorem}}}
\newreptheorem{lemma}{Lemma'}
\newreptheorem{definition}{Definition'}
\newreptheorem{proposition}{Proposition'}
\newreptheorem{theorem}{Theorem'}

\theoremstyle{plain}
\newtheorem{theorem}{Theorem}
\newtheorem{lemma}[theorem]{Lemma}

\newtheorem{assumption}{Assumption}

\newtheorem{remark}{Remark}
\newtheorem{corollary}[theorem]{Corollary}
\theoremstyle{definition}
\newtheorem{definition}{Definition}
\newcommand{\unaryminus}{\scalebox{0.85}[0.85]{\( - \)}}
\newcommand{\unaryplus}{\scalebox{0.85}[0.85]{\( + \)}}

\errorcontextlines\maxdimen

\makeatletter
    \newcommand*{\algrule}[1][\algorithmicindent]{\makebox[#1][l]{\hspace*{.5em}\thealgruleextra\vrule height \thealgruleheight depth \thealgruledepth}}%
\newcommand*{\thealgruleextra}{}
\newcommand*{\thealgruleheight}{.75\baselineskip}
\newcommand*{\thealgruledepth}{.25\baselineskip}

\newcount\ALG@printindent@tempcnta
\def\ALG@printindent{%
    \ifnum \theALG@nested>0%
        \ifx\ALG@text\ALG@x@notext%
        \else
            \unskip
            \addvspace{-1pt}%
            \ALG@printindent@tempcnta=1
            \loop
                \algrule[\csname ALG@ind@\the\ALG@printindent@tempcnta\endcsname]%
                \advance \ALG@printindent@tempcnta 1
            \ifnum \ALG@printindent@tempcnta<\numexpr\theALG@nested+1\relax%
            \repeat
        \fi
    \fi
    }%
\usepackage{etoolbox}
\patchcmd{\ALG@doentity}{\noindent\hskip\ALG@tlm}{\ALG@printindent}{}{\errmessage{failed to patch}}
\makeatother

\newbox\statebox
\newcommand{\myState}[1]{%
    \setbox\statebox=\vbox{#1}%
    \edef\thealgruleheight{\dimexpr \the\ht\statebox+1pt\relax}%
    \edef\thealgruledepth{\dimexpr \the\dp\statebox+1pt\relax}%
    \ifdim\thealgruleheight<.75\baselineskip
        \def\thealgruleheight{\dimexpr .75\baselineskip+1pt\relax}%
    \fi
    \ifdim\thealgruledepth<.25\baselineskip
        \def\thealgruledepth{\dimexpr .25\baselineskip+1pt\relax}%
    \fi
    \State #1%
    \def\thealgruleheight{\dimexpr .75\baselineskip+1pt\relax}%
    \def\thealgruledepth{\dimexpr .25\baselineskip+1pt\relax}%
}

\title{Federated Multi-Task Learning}

\author{Virginia Smith \\
Stanford \\
\texttt{\footnotesize{smithv@stanford.edu}} \hspace{-1em}
\And
Chao-Kai Chiang\thanks{Authors contributed equally.} \\
USC \\
\texttt{\footnotesize{chaokaic@usc.edu}} \hspace{-1em}
\And
Maziar Sanjabi$^*$ \\
USC \\
\texttt{\footnotesize{maziarsanjabi@gmail.com}} \hspace{-1em}
\And
Ameet Talwalkar \\
CMU \\
\texttt{\footnotesize{talwalkar@cmu.edu}}}

\date{}

\begin{document}

\maketitle

\begin{abstract}
Federated learning poses new statistical and systems challenges in training machine learning models over distributed networks of devices. In this work, we show that multi-task learning is naturally suited to handle the statistical challenges of this setting, and propose a novel systems-aware optimization method, \cocoamtl, that is robust to practical systems issues. Our method and theory for the first time consider issues of high communication cost, stragglers, and fault tolerance for distributed multi-task learning. The resulting method achieves significant speedups compared to alternatives in the federated setting, as we demonstrate through  simulations on real-world federated datasets.

\end{abstract}

\section{Introduction}
\label{sec:intro}

Mobile phones, wearable devices, and smart homes are just a few of the modern distributed networks generating massive amounts of data each day. Due to the growing storage and computational power of devices in these networks, it is increasingly attractive to store data locally and push more network computation to the edge. The nascent field of \textit{federated learning} explores \emph{training} statistical models directly on devices~\cite{McMahan:2017fl}. Examples of potential applications include: learning sentiment, semantic location, or activities of mobile phone users; predicting health events like low blood sugar or heart attack risk from wearable devices; or detecting burglaries within smart homes~\cite{Anguita:2013ap, pantelopoulos2010survey, rashidi2009keeping}. Following~\cite{Konecny:2015fo, Mcmahan:2017ce, Konecny:2016fl}, we summarize the unique challenges of federated learning below.

\begin{enumerate}[leftmargin=*]
\item \textbf{Statistical Challenges}: The aim in federated learning is to fit a model to data, $\{\Xv_1, \dots, \Xv_m\}$, generated by $m$ distributed nodes.  Each node, $t \in [m]$, collects data in a \textit{non-IID} manner across the network, with data on each node being generated by a distinct distribution $\Xv_t \sim P_t$. The number of data points on each node, $n_t$, may also vary significantly, and there may be an underlying structure present that captures the relationship amongst nodes and their associated distributions.

\item \textbf{Systems Challenges}: There are typically a large number of nodes, $m$, in the network, and communication is often a significant bottleneck. Additionally, the storage, computational, and communication capacities of each node may differ due to variability in hardware (CPU, memory), network connection (3G, 4G, WiFi), and power (battery level). These systems challenges, compounded with unbalanced data and statistical heterogeneity, make issues such as stragglers and fault tolerance significantly more prevalent than in typical data center environments. 

\end{enumerate}

In this work, we propose a modeling approach that differs significantly from prior work on federated learning, where the aim thus far has been to train a single global model across the network~\cite{Konecny:2015fo, Mcmahan:2017ce, Konecny:2016fl}. Instead, we address statistical challenges in the federated setting by learning separate models for each node, \{$\wv_1, \dots, \wv_m$\}. This can be naturally captured through a \textit{multi-task learning (MTL)} framework, where the goal is to consider fitting separate but related models simultaneously~\cite{Evgeniou:2004rm, Ando:2005af, Zhang:2010ac, Kumar:2012lt}. Unfortunately, current multi-task learning methods are not suited to handle the systems challenges that arise in federated learning, including high communication cost, stragglers, and fault tolerance. Addressing these challenges is therefore a key component of our work.

\subsection{Contributions}

We make the following contributions.  First, we show that MTL is a natural choice to handle statistical challenges in the federated setting. Second, we develop a novel method, \cocoamtl, to solve a general MTL problem. Our method generalizes the distributed optimization method \cocoa~\cite{Jaggi:2014cd, Ma:2015ti} in order to address systems challenges associated with network size and node heterogeneity.  Third, we provide convergence guarantees for \cocoamtl that carefully consider these unique systems challenges and provide insight into practical performance. Finally, we demonstrate the superior empirical performance of \cocoamtl with a new benchmarking suite of federated datasets.

\section{Related Work}
\label{sec:relatedwork}

\paragraph*{Learning Beyond the Data Center.} Computing SQL-like queries across distributed, low-powered nodes is a decades-long area of research that has been explored under the purview of query processing in sensor networks, computing at the edge, and fog computing \cite{Madden:2002ta, Deshpande:2005ma, Madden:2005ta, Bonomi:2012fc, Hong:2013mf, Garcia:2015ec}. Recent works have also considered training machine learning models centrally but serving and storing them locally, e.g., this is a common approach in mobile user modeling and personalization~\cite{Kuflik:2012ca,Rastegari2016,Ravi:2017}. However, as the computational power of the nodes within distributed networks grows, it is possible to do even more work locally, which has led to recent interest in federated learning.\footnote{The term \textit{on-device learning} has been used to describe both the task of model training and of model serving. Due to the ambiguity of this phrase, we exclusively use the term federated learning.} In contrast to our proposed approach, existing federated learning approaches~\cite{Konecny:2015fo, Mcmahan:2017ce, Konecny:2016fl, McMahan:2017fl} aim to learn a single global model across the data.\footnote{While not the focus of our work, we note privacy is an important concern in the federated setting, and that the privacy benefits associated with global federated learning (as discussed in~\cite{Mcmahan:2017ce}) also apply to our approach.} This limits their ability to deal with non-IID data and structure amongst the nodes. These works also come without convergence guarantees, and have not addressed practical issues of stragglers or fault tolerance, which are important characteristics of the federated setting. The work proposed here is, to the best of our knowledge, the first federated learning framework to consider these challenges, theoretically and in practice.

\paragraph*{Multi-Task Learning.}
In {multi-task learning}, the goal is to learn models for multiple related tasks simultaneously. While the MTL literature is extensive, most MTL modeling approaches can be broadly categorized into two groups based on how they capture relationships amongst tasks.
The first (e.g., \cite{Evgeniou:2004rm, argyriou2007MTLfeature, chen2011LowRsparse, kim2009MTLgraph}) assumes that a clustered, sparse, or low-rank structure between the tasks is known \textit{a priori}. A second group instead assumes that the task relationships are not known beforehand and can be learned directly from the data~(e.g., \cite{jacob2009clustered,Zhang:2010ac,gonccalves2016MTLcopula}). In this work, we focus our attention on this latter group, as task relationships may not be known beforehand in real-world settings. 
In comparison to learning a single global model, these MTL approaches can directly capture relationships amongst non-IID and unbalanced data, which makes them particularly well-suited for the statistical challenges of federated learning. We demonstrate this empirically on real-world federated datasets in Section~\ref{sec:experiments}.
However, although MTL is a natural modeling choice to address the statistical challenges of federated learning, currently proposed methods for distributed MTL (discussed below) do not adequately address the  systems challenges associated with federated learning.

\paragraph*{Distributed Multi-Task Learning.} Distributed multi-task learning is a relatively new area of research, in which the aim is to solve an MTL problem when data for each task is distributed over a network. While several recent works~\cite{Ahmed:2014sh, Mateos:2015do, Wang:2016da, Wang:2016db} have considered the issue of distributed MTL training, the proposed methods do not allow for flexibility of communication versus computation.  As a result, they are unable to efficiently handle concerns of fault tolerance and stragglers, the latter of which stems from both data and system heterogeneity. The works of \cite{Jin:2015cb} and \cite{Baytas:2016:am} allow for asynchronous updates to help mitigate stragglers, but do not address fault tolerance. Moreover,~\cite{Jin:2015cb} provides no convergence guarantees, and the convergence of~\cite{Baytas:2016:am} relies on a bounded delay assumption that is impractical for the federated setting, where delays may be significant and devices may drop out completely. Finally, \cite{Liu:2017dm} proposes a method and setup leveraging the distributed framework \cocoa~\cite{Jaggi:2014cd, Ma:2015ti}, which we show in Section~\ref{sec:convergence} to be a special case of the more general approach in this work. However, the authors in~\cite{Liu:2017dm} do not explore the federated setting, and their assumption that the same amount of work is done locally on each node is prohibitive in federated settings, where unbalance is common due to data and system variability. \vspace{.5em}

\section{Federated Multi-Task Learning}
\label{sec:method}

In federated learning, the aim is to learn a model over data that resides on, and has been generated by, $m$ distributed nodes.  As a running example, consider learning the activities of mobile phone users in a cell network based on their individual sensor, text, or image data. Each node (phone), $t \in [m]$, may generate data via a distinct distribution, and so it is natural to fit separate models, $\{ \wv_1, \dots, \wv_m \}$, to the distributed data---one for each local dataset. However, structure between models frequently exists (e.g., people may behave similarly when using their phones), and modeling these relationships via \textit{multi-task learning} is a natural strategy to improve performance and boost the effective sample size for each node~\cite{caruana:1998ml,Ando:2005af,Argyriou:2008cm}. In this section, we suggest a general MTL framework for the federated setting, and propose a novel method, \cocoamtl, to handle the systems challenges of federated MTL.

\subsection{General Multi-Task Learning Setup}
\label{sec:setup}
Given data $\Xv_t \in \R^{d \times n_t}$ from $m$ nodes, multi-task learning fits separate weight vectors $\wv_t \in \R^d$ to the data for each task (node) through arbitrary convex loss functions $\ell_t$ (e.g., the hinge loss for SVM models). Many MTL problems can be captured via the following general formulation:

\vspace{-2mm}
\begin{equation}
\label{eq:obj}
\min_{\Wv, \Ov} \, \left\{ \sum_{t=1}^m \sum_{i=1}^{n_t} \ell_t(\wv_t^T \xv_t^i, y_t^i) +  \bR(\Wv, \Ov) \right\} \, ,
\end{equation}
\vspace{-.5mm}

where $\Wv := [\wv_1, \dots, \wv_m] \in\R^{d\times m}$ is a matrix whose $t$-th column is the weight vector for the $t$-th task. The matrix $\Ov \in \R^{m\times m}$ models relationships amongst tasks, and is either known a priori or estimated while simultaneously learning task models. MTL problems differ based on their assumptions on $\bR$, which takes $\Ov$ as input and promotes some suitable structure amongst the tasks.

As an example, several popular MTL approaches assume that tasks form clusters based on whether or not they are related~\cite{Evgeniou:2004rm, jacob2009clustered, Zhang:2010ac, zhou2011clustered}. This can be expressed via the following bi-convex formulation:
\begin{equation}
\label{eq:fixed_cluster}
\bR(\Wv,\Ov)  = \lambda_1~\trace{\Wv\Ov\Wv^T} + \lambda_2\|\Wv\|_F^2,
\end{equation}
with constants $\lambda_1, \lambda_2>0$, and where the second term performs $L_2$ regularization on each local model. We use a similar formulation~\eqref{eq:prob} in our experiments in Section~\ref{sec:experiments}, and provide details on other common classes of MTL models that can be formulated via \eqref{eq:obj} in Appendix~\ref{sec:MTLapp}.

\subsection{\cocoamtl: A Framework for Federated Multi-Task Learning}
\label{sec:cocoamtl}

In the federated setting, the aim is to train statistical models directly on the edge, and thus we solve~\eqref{eq:obj} while assuming that the data $\{\Xv_1, \dots, \Xv_m\}$ is distributed across $m$ nodes or devices. 
Before proposing our federated method for solving \eqref{eq:obj}, we make the following observations:
\emph{
\begin{itemize}[leftmargin=*]
\item \textbf{Observation 1}: In general, \eqref{eq:obj} is not jointly convex in $\Wv$ and $\Ov$, and even in the cases where \eqref{eq:obj} is convex, solving for $\Wv$ and $\Ov$ simultaneously can be difficult~\cite{Argyriou:2008cm}. 
\item \textbf{Observation 2}: When fixing $\Ov$, updating $\Wv$ depends on both the data $\Xv$, which is distributed across the nodes, and the structure $\Ov$, which is known centrally.
\item \textbf{Observation 3}: When fixing $\Wv$, optimizing for $\Ov$ only depends on $\Wv$ and not on the data $\Xv$.
\end{itemize}
}
Based on these observations, it is natural to propose an alternating optimization approach to solve problem \eqref{eq:obj}, in which at each iteration we fix either $\Wv$ or $\Ov$ and optimize over the other, alternating until convergence is reached. Note that solving for $\Ov$ is not dependent on the data and therefore can be computed centrally; as such, we defer to prior work for this step~\cite{zhou2011clustered, jacob2009clustered,Zhang:2010ac,gonccalves2016MTLcopula}. In Appendix \ref{sec:MTLapp}, we discuss updates to $\Omega$ for several common MTL models.

In this work, we focus on developing an efficient distributed optimization method for the $\Wv$ step. In traditional data center environments, the task of distributed training is a well-studied problem, and various communication-efficient frameworks have been recently proposed, including  the state-of-the-art primal-dual \cocoa framework~\cite{Jaggi:2014cd, Ma:2015ti}.  Although \cocoa can be extended directly to update $\Wv$ in a distributed fashion across the nodes, it cannot handle the unique systems challenges of the federated environment, such as stragglers and fault tolerance, as discussed in Section~\ref{ssec:syschal}. To this end, we extend \cocoa and propose a new method, \cocoamtl, for federated multi-task learning. Our method is given in Algorithm~\ref{alg:mocha} and described in detail in Sections~\ref{sec:dualProblem} and \ref{ssec:syschal}.

\setlength{\textfloatsep}{13pt}

\begin{algorithm}[t]
\caption{\cocoamtl: Federated Multi-Task Learning Framework}
\label{alg:mocha}
\begin{algorithmic}[1]
\State {\bf Input:} Data $\Xv_t$ from $t=1, \dots, m$ tasks, stored on one of $m$ nodes, and initial matrix $\Ov_0$
\myState Starting point $\vc{\alphav}{0} := \0 \in \R^n$, $\vc{\vv}{0}:=\0\in \R^b$
\For{{\bf iterations} $i=0,1,\ldots$} \vspace{.2em}

\myState Set subproblem parameter $\sigma'$ and number of federated iterations, $H_i$

\For {{\bf iterations} $h = 0,1, \cdots, H_i $} 
 
  \For {{\bf tasks} $t \in \{1,2,\dots, m\}$ {\bf in parallel over $m$ nodes}}
     \myState call local solver, returning $\theta_t^h$-approximate solution
     $\Delta\alphav_t$
        of  the local subproblem~\eqref{eq:subproblem}
     \myState update local variables $\alphav_t \leftarrow {\alphav}_{t}+ {\Delta \alphav}_t$
     \myState return updates $\Delta \vv_t :=  \Xv_t \Delta\alphav_t$
  \EndFor \vspace{.1em}
  \myState {\textbf{reduce:}} $\vv_t  \leftarrow \vv_t +
  \textstyle \Delta \vv_t $ 
  
\EndFor \vspace{.1em}

\myState Update $\Ov$ centrally based on $\wv(\alphav)$ for latest $\alphav$

\EndFor
\State Central node computes $\wv=\wv(\alphav)$ based on the lastest $\alphav$
\State \bf{return:} $\Wv := [\wv_1, \dots, \wv_m]$
\end{algorithmic}
\end{algorithm}

\subsection{Federated Update of $\Wv$}
\label{sec:dualProblem}
To update $\Wv$ in the federated setting, we begin by extending works on distributed primal-dual optimization~\cite{Jaggi:2014cd, Ma:2015ti, Liu:2017dm} to apply to the generalized multi-task framework~\eqref{eq:obj}. This involves deriving the appropriate dual formulation, subproblems, and problem parameters, as we detail below.

\vspace{-.5em}

\paragraph{Dual problem.}

Considering the dual formulation of~\eqref{eq:obj} will allow us to better separate the global problem into distributed subproblems for federated computation across the nodes. Let $n \eqdef \sum_{t=1}^m n_t$ and $\Xv \eqdef \mbox{Diag}(\Xv_1,\cdots,\Xv_m)\in\mathbb{R}^{md\times n}$. With $\Ov$ fixed, the dual of problem~\eqref{eq:obj}, defined with respect to dual variables $\alphav \in \R^n$, is given by:
\begin{equation}
\label{eq:objDual}
\min_{ \alphav} \, \left\{ \bD(\alphav) \eqdef \sum_{t=1}^m \sum_{i=1}^{n_t} \ell_t^*(-\alphav_t^i) +  \bR^*(\X\alphav) \right\} \, ,
\end{equation}
where $\ell_t^*$ and $\bR^*$ are the conjugate dual functions of $\ell_t$ and $\bR$, respectively, and $\alphav_t^i$ is the dual variable for the data point $(\xv_t^i,y_t^i)$.
Note that $\bR^*$ depends on $\Ov$, but for the sake of simplicity, we have removed this in our notation. To derive distributed subproblems from this global dual, we make an assumption described below on the regularizer $\bR$.

\vspace{1mm}
\begin{assumption}
\label{asm:strong}
Given $\Ov$, we assume that there exists a symmetric positive definite matrix $\Mv\in\mathbb{R}^{md\times md}$, depending on $\Ov$, for which the function $\bR$ is strongly convex with respect to $\Mv^{-1}$. Note that this corresponds to assuming that $\bR^*$ will be smooth with respect to matrix $\Mv$. 
\end{assumption}
\vspace{1mm}

\begin{remark}
\label{rmk:vecrep}
We can reformulate the MTL regularizer in the form of $\bar{\bR}(\wv,\bar{\Ov}) = \bR(\Wv,\Ov)$, where $\wv\in\mathbb{R}^{md}$ is a vector containing the columns of $\Wv$ and $\bar{\Ov} \eqdef \Ov \otimes \mathbf{I}_{d\times d} \in \mathbb{R}^{md\times md}$. For example, we can rewrite the regularizer in \eqref{eq:fixed_cluster} as $\bar{\bR}(\wv,\bar{\Ov}) = \trace{\wv^T(\lambda_1\bar{\Ov}+\lambda_2\Iv)\wv}$. Writing the regularizer in this form, it is clear that it is strongly convex with respect to matrix $\Mv^{-1} = \lambda_1\bar{\Ov}+\lambda_2\Iv$.
\end{remark}

\paragraph{Data-local quadratic subproblems.} To solve~\eqref{eq:obj} across distributed nodes, we define the following data-local subproblems, which are formed via a careful quadratic approximation of the dual problem \eqref{eq:objDual} to separate computation across the nodes. These subproblems find updates $\Delta\alphav_t\in\mathbb{R}^{n_t}$ to the dual variables in $\alphav$ corresponding to a single node $t$, and only require accessing data which is available locally, i.e., $\X_t$ for node $t$.  The $t$-th subproblem is given by:
\begin{align}
\min_{\Delta \alphav_t } \
\Ggk( \Delta \alphav_t; \vv_t, \alphav_t):= \sum_{i =1}^{n_t} \ell^*_t(\unaryminus\alphav_t^i \unaryminus\Delta\alphav_t^i)
\unaryplus \langle \wv_t(\alphav), \X_t\Delta\alphav_t\rangle
\unaryplus \frac{\sigma'}{2}  \left\| \X_t\Delta\alphav_t \right\|_{\Mv_t}^2\unaryplus c(\alphav) \,,
\label{eq:subproblem}
\end{align}
where $c(\alphav) := \frac{1}{m} \bR^*(\Xv\alphav)$, and $\Mv_t\in\mathbb{R}^{d\times d}$ is the $t$-th diagonal block of the symmetric positive definite matrix $\Mv$. Given dual variables $\alphav$,  corresponding primal variables can be found via $\wv(\alphav) = \nabla\bR^*(\Xv\alphav)$, where $\wv_t(\alphav)$ is the $t$-th block in the vector $\wv(\alphav)$. Note that computing $\wv(\alphav)$ requires the vector $\vv = \Xv\alphav$. The $t$-th block of $\vv$, $\vv_t \in \R^d$, is the only information that must be communicated between nodes at each iteration. 
Finally, $\sigma'>0$ measures the difficulty of the data partitioning, and helps to relate progress made to the subproblems to the global dual problem. It can be easily selected based on $\Mv$ for many applications of interest; we provide details in Lemma \ref{lemma:sigma} of the Appendix.

\subsection{Practical Considerations}
\label{ssec:syschal}

During \cocoamtl's federated update of $\Wv$, the central node requires a response from all workers before
performing a synchronous update. In the federated setting, a naive execution of this communication protocol could introduce dramatic straggler effects  due to node heterogeneity. To avoid stragglers,  \cocoamtl provides the $t$-th node with the flexibility to \textit{approximately solve} its subproblem $\Ggk(\cdot)$, where the quality of the approximation is controled by a per-node parameter $\theta_t^h$.  The following factors determine the quality of the $t$-th node's solution to its subproblem:
\begin{enumerate}[leftmargin=*]
\item \textbf{Statistical challenges}, such as the size of $\Xv_t$ and the intrinsic difficulty of subproblem $\Ggk(\cdot)$.
\item \textbf{Systems challenges}, such as the node's storage, computational, and  communication capacities due to hardware (CPU, memory), network connection (3G, 4G, WiFi), and power (battery level).
\item A \textbf{global clock cycle} imposed by the central node specifying a deadline for receiving updates.
\end{enumerate}
We define $\theta_t^h$ as a function of these factors, and assume that each node has a controller that may derive $\theta_t^h$ from the current clock cycle and statistical/systems setting.  $\theta_t^h$ ranges from zero to one, where $\theta_t^h=0$ indicates an exact solution to $\Ggk(\cdot)$ and $\theta_t^h=1$ indicates that node $t$ made no progress during iteration $h$ (which we refer to as a \emph{dropped node}). For instance, a node may `drop' if it runs out of battery, or if its network bandwidth deteriorates during iteration $h$ and it is thus unable to return its update within the current clock cycle. A formal definition of $\theta_t^h$ is provided in~\eqref{eq:localSolutionQuality2a} of Section~\ref{sec:convergence}.

\cocoamtl mitigates stragglers by enabling the $t$-th node to define its own $\theta_t^h$. On every iteration $h$, the local updates that a node performs and sends in a clock cycle will yield a specific value for $\theta_t^h$.  As discussed in Section~\ref{sec:convergence}, \cocoamtl is additionally robust to a small fraction of nodes periodically dropping and performing no local updates (i.e., $\theta_t^h := 1$) under suitable conditions, as defined in Assumption~\ref{asm:prob_stragglers}. In contrast, prior work of \cocoa may suffer from severe straggler effects in federated settings, as it requires a \emph{fixed $\theta_t^h=\theta$ across all  nodes and all iterations} while still maintaining synchronous updates, and it does not allow for the case of dropped nodes ($\theta := 1$).

Finally, we note that asynchronous updating schemes are an alternative approach to mitigate stragglers.  We do not consider these approaches in this work, in part due to the fact that the bounded-delay assumptions associated with most asynchronous schemes limit fault tolerance.  However, it would be interesting to further explore the differences and connections between asynchronous methods and approximation-based, synchronous methods like \cocoamtl in future work.

\section{Convergence Analysis}
\label{sec:convergence}
\cocoamtl is based on a bi-convex alternating approach, which is guaranteed to converge \cite{gorski2007biconvex,razaviyayn2013BSUM} to a stationary solution of problem \eqref{eq:obj}. In the case where this problem is jointly convex with respect to $\Wv$ and $\Ov$, such a solution is also optimal. In the rest of this section, we therefore focus on the convergence of solving the $\Wv$ update of \cocoamtl in the federated setting. Following the discussion in Section~\ref{ssec:syschal}, we first introduce the following per-node, per-round approximation parameter.

\begin{definition}[Per-Node-Per-Iteration-Approximation Parameter]
\label{def:approx}

At each iteration $h$, we define the accuracy level of the solution calculated by node $t$ to its subproblem~\eqref{eq:subproblem} as:
\begin{align}
\label{eq:localSolutionQuality2a}
\theta_t^h
\eqdef
\frac{ \Ggk( \Delta \alphav_t^{(h)}; \vv^{(h)}, \alphav_t^{(h)} ) - \Ggk( \Delta \alphav_t^\star; \vv^{(h)}, \alphav_t^{(h)} ) }{ \Ggk( {\bf 0}; \vv^{(h)}, \alphav_t^{(h)} ) - \Ggk( \Delta \alphav_t^\star; \vv^{(h)}, \alphav_t^{(h)} ) } \, ,
\end{align}
where
$\Delta \alphav_t^\star$ is the minimizer of subproblem $\Ggk(\cdot \ ; \vv^{(h)}, \alphav_t^{(h)} )$.
We allow this value to vary between $[0,1]$, with $\theta_t^h := 1$ meaning that no updates to subproblem $\Ggk$ are made by node $t$ at iteration $h$. 
\end{definition}

While the flexible per-node, per-iteration approximation parameter $\theta_t^h$ in \eqref{eq:localSolutionQuality2a} allows the consideration of stragglers and fault tolerance, these additional degrees of freedom also pose new challenges in providing convergence guarantees for \cocoamtl. We introduce the following assumption on $\theta_t^h$ to provide our convergence guarantees. 

\begin{assumption}
\label{asm:prob_stragglers}

Let $\bH_h \eqdef (\vc{\alphav}{h}, \vc{\alphav}{h-1}, \cdots, \vc{\alphav}{1})$ be the dual vector history until the beginning of iteration $h$, and define $\Theta_t^h \eqdef \E [ \theta_t^h | \bH_h]$. For all tasks $t$ and all iterations $h$, we assume $p_t^{h} \eqdef \Prob[\theta^h_t = 1] \leq p_{\max}<1$ and $\hat{\Theta}_t^h \eqdef \mathbb{E}[\theta_t^h|\mathcal{H}_h, \theta_t^h<1]\leq\Theta_{\max}<1$.

\end{assumption}
This assumption  states that at each iteration, the \textit{probability} of a node sending a result is non-zero, and that the quality of the returned result is, on average, better than the previous iterate. Compared to~\cite{Smith16,Liu:2017dm} which assumes $\theta_t^h = \theta<1$, our assumption is significantly less restrictive and better models the federated setting, where nodes are unreliable and may periodically drop out.

Using Assumption~\ref{asm:prob_stragglers}, we derive the following theorem, which characterizes the convergence of the federated update of \cocoamtl in finite horizon when the losses $\ell_t$ in \eqref{eq:obj} are smooth.

\begin{theorem}

\label{thm:convergenceSmoothCasePart1}

Assume that the losses $\ell_t$ are $(1/\mu)$-smooth. Then, under Assumptions \ref{asm:strong} and \ref{asm:prob_stragglers}, there exists a constant $s\in(0,1]$ such that for any given convergence target $\epsilon_{\bD}$, choosing $H$ such that
\begin{equation}
H
\geq
\frac{1}{(1-\bar\Theta) s}\log \frac{n}{\epsilon_{\bD}} \, ,
\end{equation}
will satisfy~$\Exp[\bD(\vc{\alphav}{H})-\bD(\alphav^{\star})]\leq\epsilon_{\bD}$ .

\end{theorem}
Here, $\bar{\Theta}\eqdef p_{\max} + (1-p_{\max})\Theta_{\max}<1$. While Theorem \ref{thm:convergenceSmoothCasePart1} is concerned with finite horizon convergence, it is possible to get asymptotic convergence results, i.e., $H\rightarrow\infty$, with milder assumptions on the stragglers; see Corollary \ref{coro:asym} in the Appendix for details.

When the loss functions are non-smooth, e.g., the hinge loss for SVM models, we provide the following sub-linear convergence for $L$-Lipschitz losses.
\begin{theorem}
\label{thm:LipschitzLosses}
If the loss functions $\ell_t$ are $L$-Lipschitz, then there exists a constant $\sigma$, defined in \eqref{eq:sigma}, such that for any given $\epsilon_{\bD}>0$, if we choose
\begin{align}
\label{eq:sublinear}
& H \geq H_0 + \bigg\lceil \frac{2}{(1-\bar{\Theta})}\max\left(1, \frac{2L^2\sigma\sigma'}{n^2\epsilon_{\bD}}\right)\bigg\rceil, \\
 \text{with } H_0 \geq \bigg\lceil h_0 + & \frac{16L^2\sigma\sigma'}{(1-\bar{\Theta})n^2\epsilon_{\bD}}\bigg\rceil, h_0 = \left[1+\frac{1}{(1-\bar{\Theta})}\log\left(\frac{2n^2(D(\alphav^{\star})-D(\alphav^{0}))}{4L^2\sigma\sigma'}\right)\right]_{+}\nonumber ,
\end{align}
then $\bar{\alphav}\eqdef \frac{1}{H-H0}\sum_{h=H_0+1}^H\alphav^{(h)}$ will satisfy $\mathbb{E}[\bD(\bar{\alphav})-\bD(\alphav^{\star})]\leq \epsilon_{\bD}$ .

\end{theorem}

These theorems guarantee that \cocoamtl will converge in the federated setting, under mild assumptions on stragglers and capabilities of the nodes. While these results consider convergence in terms of the dual, we show that they hold analogously for the duality gap. We provide all proofs in Appendix~\ref{Apdx:ConvergenceAnalysis}.

\begin{remark}
Following from the discussion in Section~\ref{ssec:syschal}, our method and theory generalize the results in \cite{Jaggi:2014cd, Ma:2015ti}. In the limiting case that all $\theta_t^h$ are identical, our results extend the results of \cocoa to the multi-task framework described in~\eqref{eq:obj}.
\end{remark}
\begin{remark}
Note that the methods in \cite{Jaggi:2014cd, Ma:2015ti} have an aggregation parameter $\aggpar\in(0,1]$. Though we prove our results for a general $\aggpar$, we simplify the method and results here by setting $\aggpar:=1$, which has been shown to have the best performance, both theoretically and empirically~\cite{Ma:2015ti}.
\end{remark}

\section{Simulations}
\label{sec:experiments}

In this section we validate the empirical performance of \cocoamtl. First, we introduce a benchmarking suite of real-world federated datasets and show that multi-task learning is well-suited to handle the statistical challenges of the federated setting. Next, we demonstrate \cocoamtl's ability to handle stragglers, both from statistical and systems heterogeneity. Finally, we explore the performance of \cocoamtl when devices periodically drop out. Our code is available at: \href{{https://github.com/gingsmith/fmtl}}{\tt{github.com/gingsmith/fmtl}}.

\subsection{Federated Datasets}
\label{sec:data}

In our simulations, we use several real-world datasets that have been generated in federated settings. We provide additional details in the Appendix, including information about data sizes, $n_t$.

\begin{itemize}[leftmargin=*]
\item \textbf{Google Glass (GLEAM)\footnote{\scriptsize\url{http://www.skleinberg.org/data/GLEAM.tar.gz}}}: This dataset consists of two hours of high resolution  sensor data collected from 38 participants wearing Google Glass for the purpose of activity recognition. Following~\cite{Rahman:2015ue}, we featurize the raw accelerometer, gyroscope, and magnetometer data into 180 statistical, spectral, and temporal features. We model each participant as a separate task, and predict between eating and other activities (e.g., walking, talking, drinking).
\item \textbf{Human Activity Recognition\footnote{\scriptsize\url{https://archive.ics.uci.edu/ml/datasets/Human+Activity+Recognition+Using+Smartphones}}}: Mobile phone accelerometer and gyroscope data collected from 30 individuals, performing one of six activities: \{\textit{walking, walking-upstairs, walking-downstairs, sitting, standing, lying-down}\}.
We use the provided 561-length feature vectors of time and frequency domain variables generated for each instance~\cite{Anguita:2013ap}.
We model each individual as a separate task and predict between sitting and the other activities.

\item \textbf{Vehicle Sensor\footnote{\scriptsize\url{http://www.ecs.umass.edu/~mduarte/Software.html}}}: Acoustic, seismic, and infrared sensor data collected from a distributed network of 23 sensors, deployed with the aim of classifying vehicles driving by a segment of road~\cite{Duarte:2004vc}. Each instance is described by 50 acoustic and 50 seismic features. We model each sensor as a separate task and predict between AAV-type and DW-type vehicles.
\end{itemize}

\subsection{Multi-Task Learning for the Federated Setting}

We demonstrate the benefits of multi-task learning for the federated setting by comparing the error rates of a multi-task model to that of a fully local model (i.e., learning a model for each task separately) and a fully global model (i.e., combining the data from all tasks and learning one single model). Work on federated learning thus far has been limited to the study of fully global models~\cite{Konecny:2015fo, Mcmahan:2017ce, Konecny:2016fl}.

 We use a cluster-regularized multi-task model~\cite{Zhang:2010ac}, as described in Section~\ref{sec:setup}. For each dataset from Section~\ref{sec:data}, we randomly split the data into 75\% training and 25\% testing, and learn multi-task, local, and global support vector machine models, selecting the best regularization parameter, $\lambda \in $\{1e-5, 1e-4, 1e-3, 1e-2, 0.1, 1, 10\}, for each model using 5-fold cross-validation. We repeat this process 10 times and report the average prediction error across tasks, averaged across these 10 trials.

\begin{table}[h]
\centering
\caption{Average prediction error: Means and standard errors over 10 random shuffles.}
\vspace{1em}
{\small
\begin{tabular}{c c c c  c}
\toprule
{\textbf{Model}} & Human Activity & Google Glass  & Vehicle Sensor  \\ \toprule

Global & 2.23 (0.30) & 5.34 (0.26)  &  13.4 (0.26) \\
\midrule
Local & 1.34 (0.21) & 4.92 (0.26)  &  7.81 (0.13)  \\
\midrule
MTL & \textbf{0.46 (0.11)} & \textbf{2.02 (0.15)}  & \textbf{6.59 (0.21)}  \\
\bottomrule
\end{tabular}
}
\label{table:mtl}
\end{table}
\vspace{1em}

In Table~\ref{table:mtl}, we see that for each dataset, multi-task learning significantly outperforms the other models in terms of achieving the lowest average error across tasks. The global model, as proposed in~\cite{Konecny:2015fo, Mcmahan:2017ce, Konecny:2016fl} performs the worst, particularly for the Human Activity and Vehicle Sensor datasets. Although the datasets are already somewhat unbalanced, we note that a global modeling approach may benefit tasks with a very small number of instances, as information can be shared across tasks. For this reason, we additionally explore the performance of global, local, and multi-task modeling for highly skewed data in Table~\ref{table:skew} of the Appendix. Although the performance of the global model improves slightly relative to local modeling in this setting, the global model still performs the worst for the majority of the datasets, and MTL still significantly outperforms both global and local approaches.

\subsection{Straggler Avoidance}

Two challenges that are prevalent in federated learning are stragglers and high communication. Stragglers can occur when a subset of the devices take much longer than others to perform local updates, which can be caused either by statistical or systems heterogeneity. Communication can also exacerbate poor performance, as it can be slower than computation by many orders of magnitude in typical cellular or wireless networks~\cite{van2009multi, Huang2013an, singelee2011communication, Carroll2010aa, Miettinen2010ee}. 
In our experiments below, we simulate the time needed to run each method by tracking the operations and communication complexities, and scaling the communication cost relative to computation by one, two, or three orders of magnitude, respectively. These numbers correspond roughly to the clock rate vs. network bandwidth/latency (see, e.g.,~\cite{van2009multi}) for modern cellular and wireless networks. Details are provided in Appendix~\ref{sec:expdetails}. \vspace{.5em}

\paragraph{Statistical Heterogeneity.} We explore the effect of statistical heterogeneity on stragglers for various methods and communication regimes (3G, LTE, WiFi). For a fixed communication network, we compare \cocoamtl to \cocoa, which has a single $\theta$ parameter, and to mini-batch stochastic gradient descent (Mb-SGD) and mini-batch stochastic dual coordinate ascent (Mb-SDCA), which have limited communication flexibility depending on the batch size. We tune all compared methods for best performance, as we detail in Appendix~\ref{sec:expdetails}. In Figure~\ref{fig:stath}, we see that while the performance degrades for mini-batch methods in high communication regimes, \cocoamtl and \cocoa are robust to high communication. However, \cocoa is significantly affected by stragglers---because $\theta$ is fixed across nodes and rounds, difficult subproblems adversely impact convergence. In contrast, \cocoamtl performs well regardless of communication cost and is robust to statistical heterogeneity. 

\newcommand{\smalltrimfig}[1]{\subfigure{\includegraphics[trim = 30 180 60 180, clip, width=.33\linewidth]{#1}}}
\begin{figure*}
\smalltrimfig{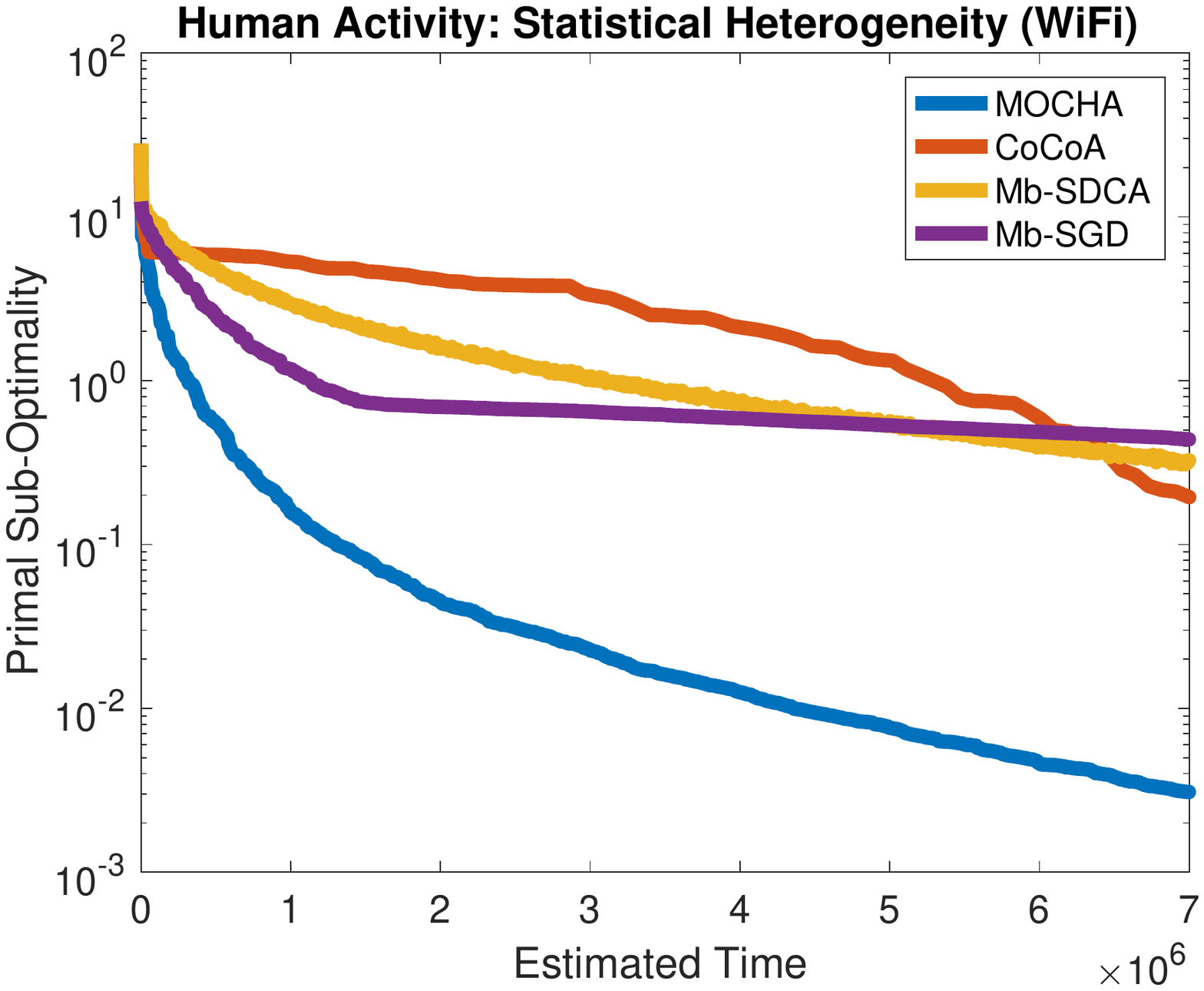}
\smalltrimfig{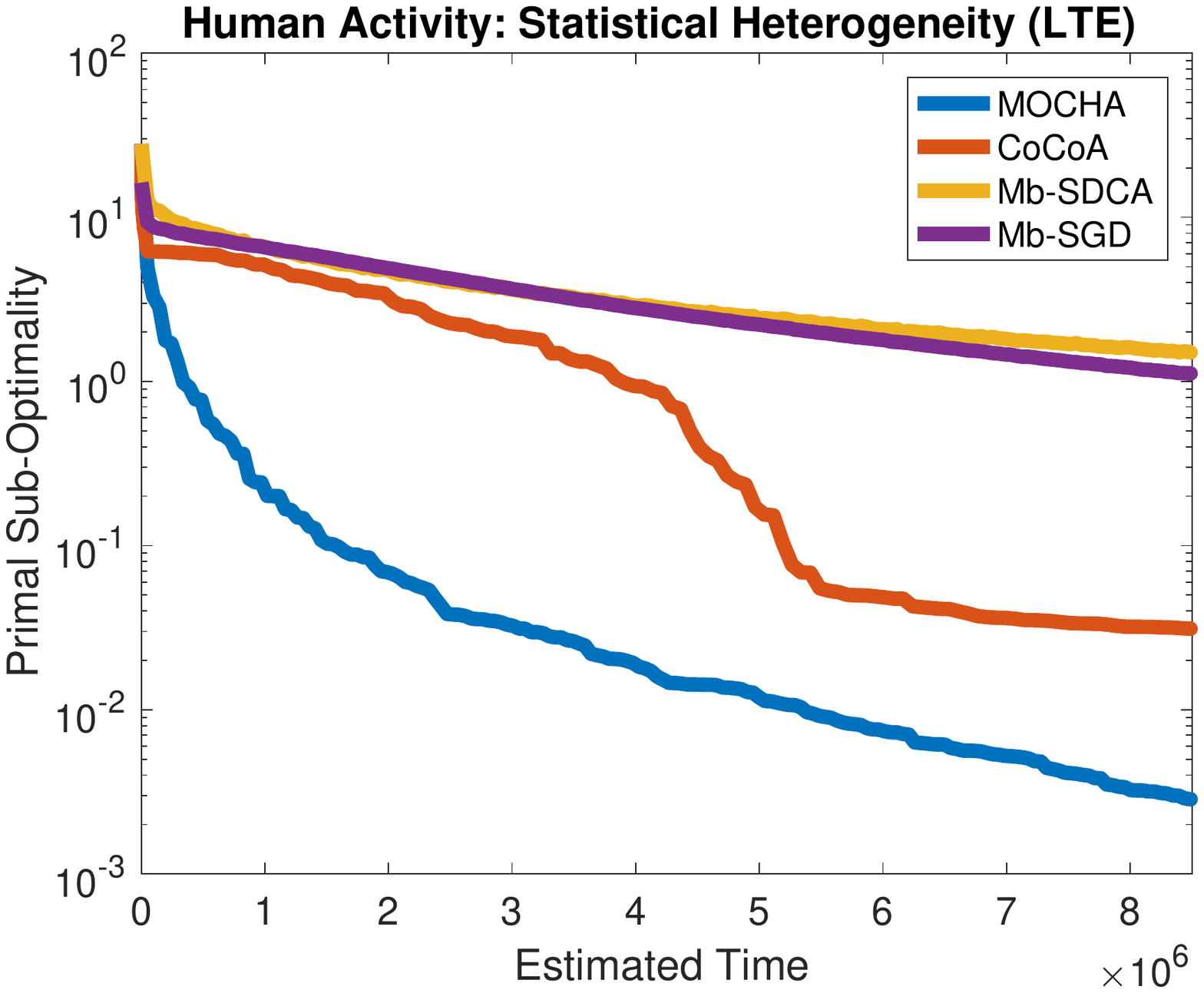}
\smalltrimfig{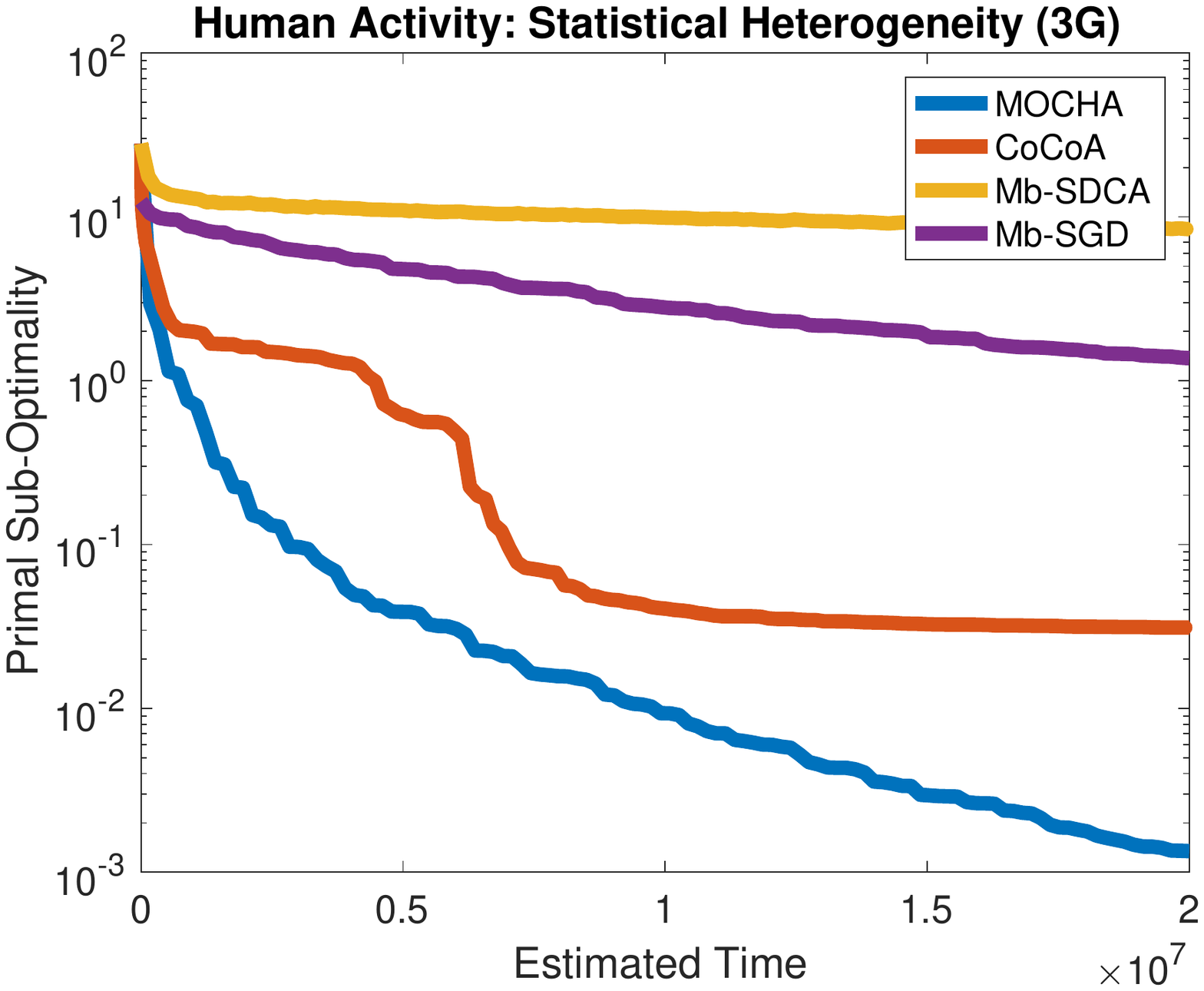}
\caption{\small The performance of \cocoamtl compared to other distributed methods for the $\Wv$ update of~\eqref{eq:obj}. While increasing communication tends to \textit{decrease} the performance of the mini-batch methods, \cocoamtl performs well in high communication settings. In all settings, \cocoamtl with varied approximation values, $\Theta^h_t$, performs better than without (i.e., naively generalizing \cocoa), as it avoids stragglers from statistical heterogeneity.}
\label{fig:stath}
\end{figure*}

\paragraph{Systems Heterogeneity.} \cocoamtl is also equipped to handle heterogeneity from changing systems environments, such as battery power, memory, or network connection, as we show in Figure~\ref{fig:sysh}. In particular, we simulate systems heterogeneity by randomly choosing the number of local iterations for \cocoamtl or the mini-batch size for mini-batch methods, between $10\%$ and $100\%$ of the minimum number of local data points for high variability environments, to between $90\%$ and $100\%$ for low variability (see Appendix~\ref{sec:expdetails} for full details). We do not vary the performance of \cocoa, as the impact from statistical heterogeneity alone significantly reduces performance. However, adding systems heterogeneity would reduce performance even further, as the maximum $\theta$ value across all nodes would only increase if additional systems challenges were introduced. 

\newcommand{\wraptrimfig}[1]{\subfigure{\includegraphics[trim = 30 180 60 190, clip, width=.495\linewidth]{#1}}}
\begin{wrapfigure}{r}{.5\textwidth}
\vspace{-4em}
\begin{minipage}[t]{\linewidth}
\wraptrimfig{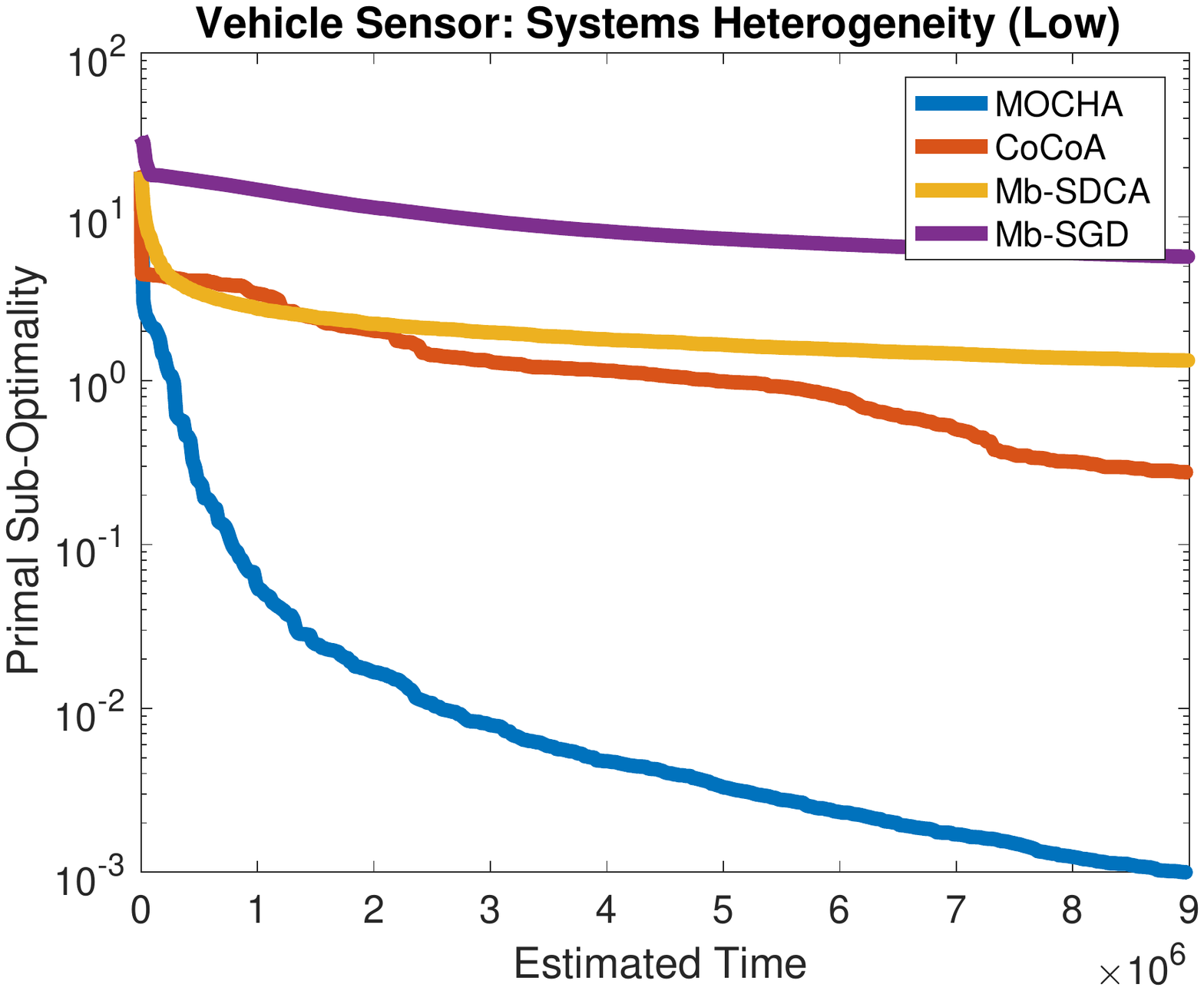}
\wraptrimfig{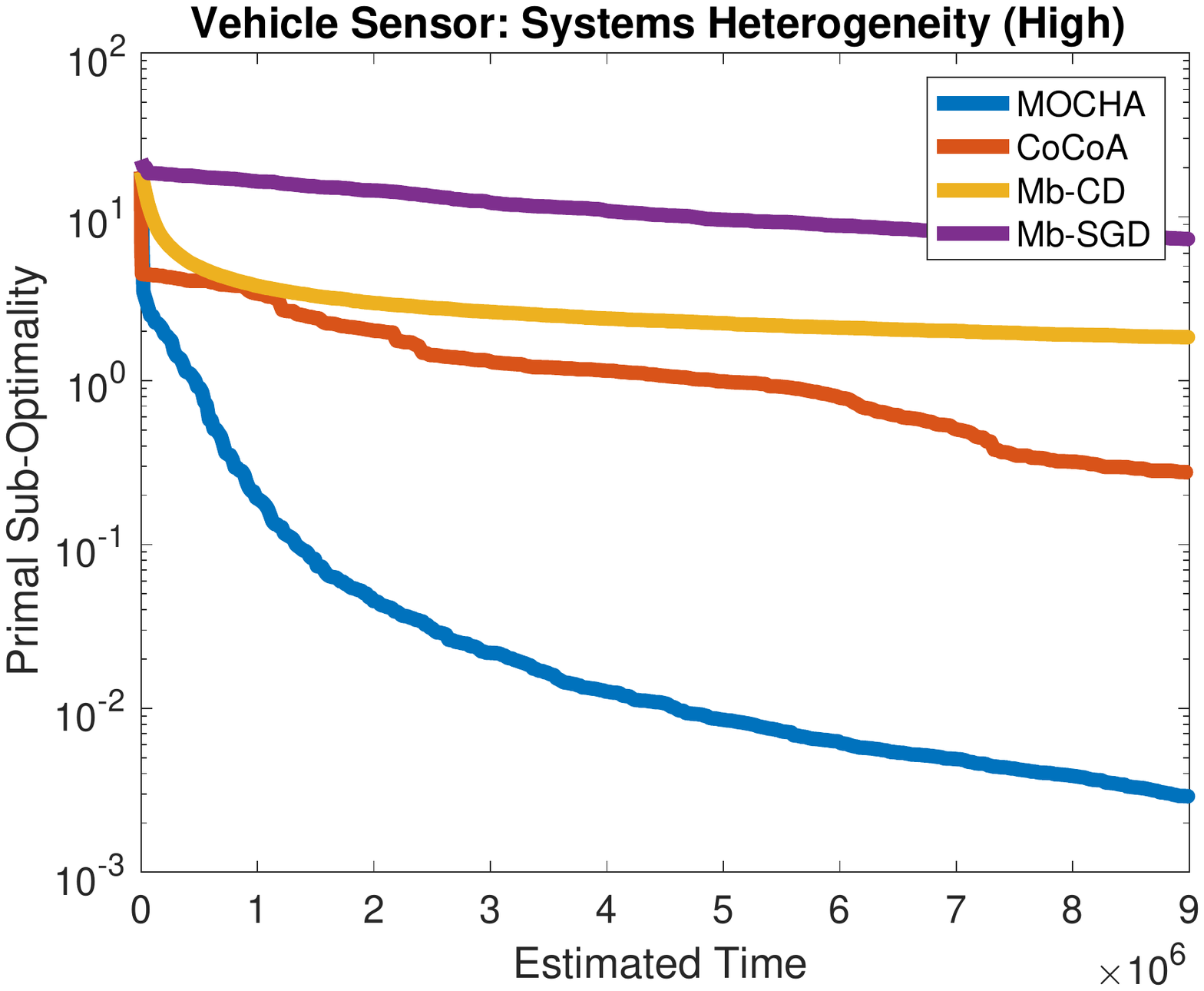}
\caption{\small \cocoamtl can handle variability from systems heterogeneity.}
    \label{fig:sysh}
\end{minipage}
   \hfill
   \vspace{1em}
\begin{minipage}[t]{\linewidth}
\wraptrimfig{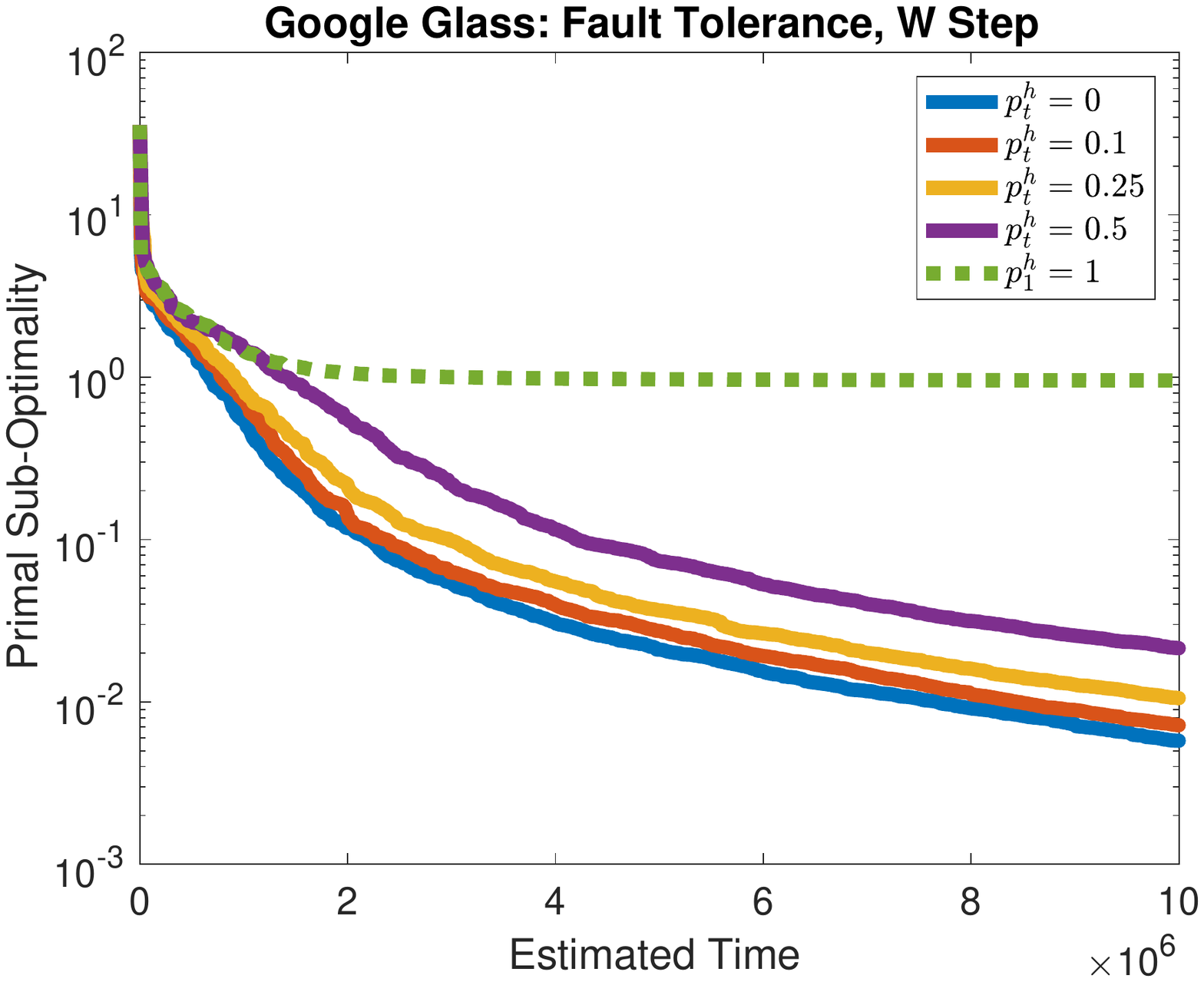}
\wraptrimfig{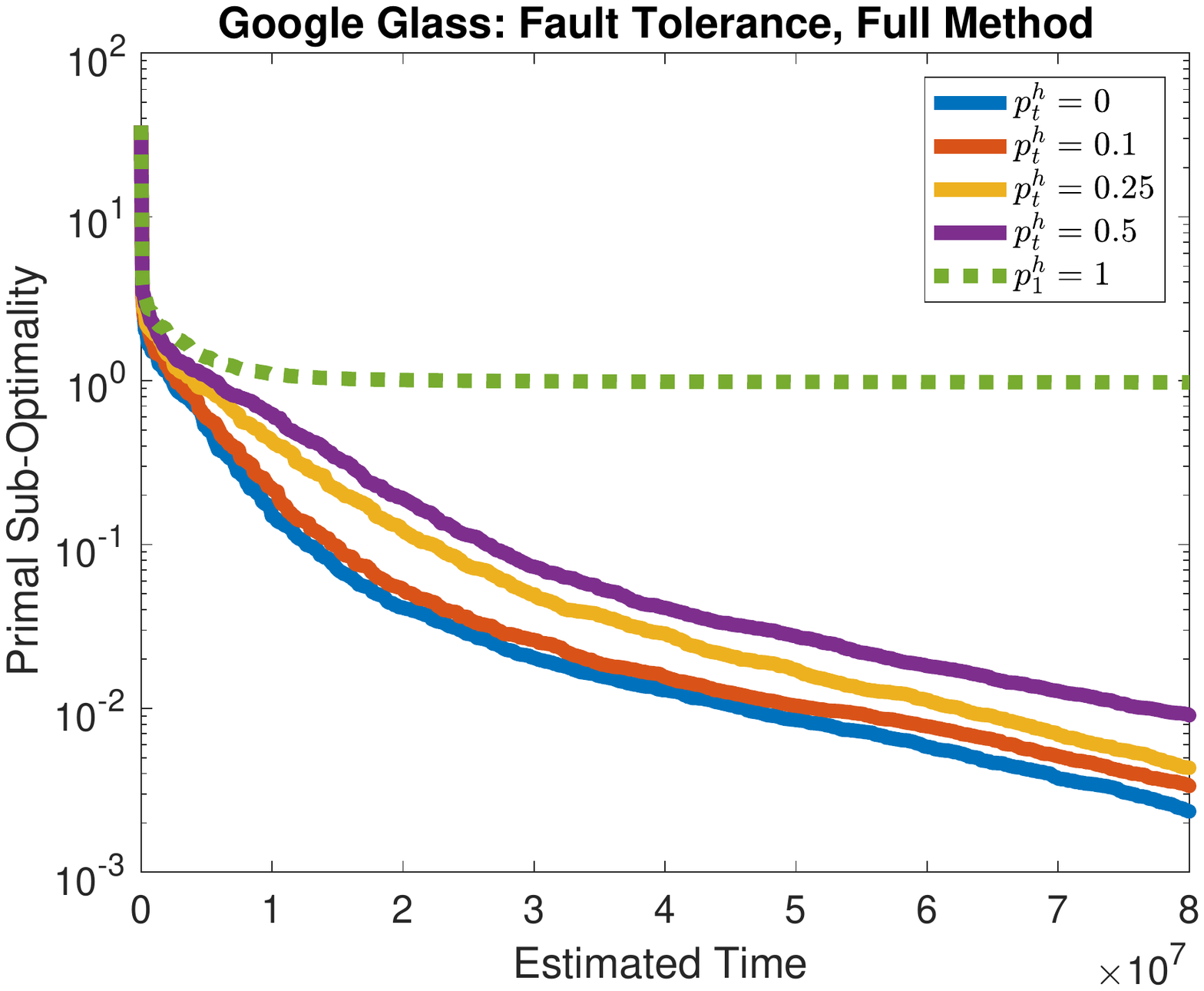}
\caption{\small The performance of \cocoamtl is robust to nodes periodically dropping out (fault tolerance).}
    \label{fig:ft}
\end{minipage}
\end{wrapfigure}

\subsection{Tolerance to Dropped Nodes}

Finally, we explore the effect of nodes dropping on the performance of \cocoamtl. We do not draw comparisons to other methods, as to the best of our knowledge, no other methods for distributed multi-task learning directly address fault tolerance. In \cocoamtl, we incorporate this setting by allowing $\theta_t^h := 1$, as explored theoretically in Section~\ref{sec:convergence}. In Figure~\ref{fig:ft}, we look at the performance of \cocoamtl, either for one fixed $\Wv$ update, or running the entire \cocoamtl method, as the probability that nodes drop at each iteration ($p_t^h$ in Assumption~\ref{asm:prob_stragglers}) increases. We see that  the performance of \cocoamtl is robust to relatively high values of $p_t^h$, both during a single update of $\Wv$ and in how this affects the performance of the overall method. However, as intuition would suggest, 
if one of the nodes \textit{never} sends updates (i.e., $p_1^h := 1$ for all $h$, green dotted line), the method does not converge to the correct solution. This provides validation for our Assumption~\ref{asm:prob_stragglers}.

\section{Discussion}
To address the statistical and systems challenges of the burgeoning federated learning setting, we have presented \cocoamtl, a novel systems-aware optimization framework for federated multi-task learning. Our method and theory for the first time consider issues of high communication cost, stragglers, and fault tolerance for multi-task learning in the federated environment.
While \cocoamtl does not apply to non-convex deep learning models in its current form, we note that there may be natural connections between this approach and ``convexified'' deep learning models~\cite{aslan2014convex,mairal2014convolutional,tsai2016tensor,zhang2016convexified} in the context of kernelized federated multi-task learning.

\section*{Acknowledgements} We thank Brendan McMahan, Chlo\'{e} Kiddon,  Jakub Kone\v{c}n\'{y}, Evan Sparks, Xinghao Pan, Lisha Li, and Hang Qi for valuable discussions and feedback.

{\small
\bibliographystyle{abbrv}
\bibliography{bibliography}
}

\newpage
\appendix

\section{Preliminaries}
\paragraph{Notation.} We use $\Iv_{d\times d}$ to represent an identity matrix of size $d\times d$. When the context allows, we use the notation $\Iv$ to denote an identity matrix of an appropriate size. We also use $\otimes$ to denote the Kronecker product between two matrices.
\begin{definition}[Matrix norm]
 Given a symmetric positive definite matrix $\mathbf{M}$, the norm of $\uv$ with respect to $\mathbf{M}$ is given by $\| \uv \|_{\mathbf{M}} \eqdef \sqrt{ \uv^T \mathbf{M} \uv }$ .
 \end{definition}
 \begin{definition}[$L$-smooth] A convex function $f$ is $L$-smooth with respect to $\mathbf{M}$ if
\begin{equation}
f(\uv) \leq f(\vv) + \langle \nabla f(\vv), \uv-\vv \rangle + \frac{L}{2} \| \uv-\vv \|_{\Mv}^2  \qquad\forall \uv,\vv \, .
\label{eq:smooth}
\end{equation}
If $\Mv=\Iv$~ then, we simply say $f$ is $L$-smooth.
\end{definition}
\begin{definition}[$\tau$-strongly convex]
A function $f$ is $\tau$-strongly convex with respect to $\Mv$ if
\begin{equation}
f(\uv) \geq f(\vv) + \langle \mathbf{z}, \uv-\vv \rangle + \frac{\tau}{2} \| \uv-\vv \|_{\Mv}^2  \qquad\forall \uv,\vv,~\mathbf{z} \in \partial f(\vv) \, ,
\label{eq:strong}
\end{equation}
where $\partial f(\vv)$ is the set of sub-differentials of function $f$ at $\vv$. If $\Mv=\Iv$ ~then, we simply say $f$ is $\tau$-strongly convex.
\end{definition}
\begin{definition}
The function $f$ is called $L$-Lipchitz if for any $\xv$ and $\yv$ in its domain
\begin{align}
\label{eq:LLip}
|f(\xv)-f(\yv)|\leq L\|\xv-\yv\|_2 \, .
\end{align}
\end{definition}
If a function $f$ is $L$-Lipchitz then its dual will be $L$-bounded, i.e., for any $\alphav$ such that $\|\alphav\|_2>L$, then $f^*(\alphav) = +\infty$.

\section{Multi-Task Learning}
\label{sec:MTLapp}
In this section, we summarize several popular multi-task learning formulations that can be written in the form of \eqref{eq:obj} and can therefore be addressed by our framework, \cocoamtl. While the $\Wv$ update is discussed in Section~\ref{sec:method}, we provide details here on how to solve the $\Omega$ update for these formulations.

\subsection{Multi-Task Learning Formulations}
\label{sec:mtlformulations}

\paragraph{MTL with cluster structure.}
In MTL models that assume a cluster structure, the weight vectors for each task, $\wv_t$, are assumed to `close' according to some metric to other weight vectors from tasks in the same cluster.
This idea goes back to mean-regularized MTL \cite{Evgeniou:2004rm}, which assumes that all the tasks form one cluster, and that the weight vectors are close to their mean. Such a regularizer could be formulated in the form of \eqref{eq:obj} by choosing $\Ov = (\Iv_{m\times m} - \frac{1}{m}\1\1^T)^2$, where $\Iv_{m\times m}$ is the identity matrix of size $m\times m$ and $\mathbf{1}_m$ represents a vector of all ones with size $m$. In this case, we set $\bR$ to be
\begin{equation}
\label{eq:fixed_cluster_app}
\bR(\Wv,\Ov) = 
\lambda_1~\trace{\Wv\Ov\Wv^T} + \lambda_2\|\Wv\|_F^2 \, ,
\end{equation}
where $\lambda_1, \lambda_2>0$ are parameters. Note that in this formulation, the structural dependence matrix $\Ov$ is known a-priori.
However, it is natural to assume multiple clusters exist, and to learn this clustering structure directly from the data \cite{zhou2011clustered}. For such a model, the problem formulation is non-convex if a perfect clustering structure is imposed \cite{zhou2011clustered, jacob2009clustered}. However, by performing a convex relaxation, the following regularizer is obtained \cite{zhou2011clustered, jacob2009clustered}
\begin{equation}
\label{eq:cluster}
\bR(\Wv,\Ov) = \lambda~\trace{\Wv(\eta \Iv+\Ov)^{-1}\Wv^T}, ~\Ov\in\mathcal{Q} = \bigg\{\mathbf{Q}~|~\mathbf{Q}\succeq \0, ~\trace{\mathbf{Q}} = k, ~\mathbf{Q}\preceq \Iv\bigg\},
\end{equation}
where $\lambda$ and $\eta$ are regularization parameters, $k$ is the number of clusters, and $\Ov$ defines the clustering structure.

\paragraph{MTL with probabilistic priors.} 
Another set of MTL models that can be realized by our framework enforce structure by putting probabilistic priors on the dependence among the columns of $\Wv$.  For example, in \cite{Zhang:2010ac} it is assumed that the weight matrix $\Wv$ has a prior distribution of the form:
\begin{equation}
\label{eq:prior}
\Wv~\sim~\left(\prod_{i=1}^m \mathcal{N}(\0, \sigma^2\Iv)\right)~\mathcal{MN}(\0,\Iv_{d\times d} \otimes \Ov) \, ,
\end{equation}
where $\mathcal{N}(\0, \sigma^2\Iv)$ denotes the normal distribution with mean $\0$ and covariance $\sigma^2\Iv$, and $\mathcal{MN}(\0,\Iv_{d\times d} \otimes \Ov)$ denotes the matrix normal distribution with mean $\0$, row covariance $\Iv_{d\times d}$, and column covariance $\Ov$. This prior generates a regularizer of the following form \cite{Zhang:2010ac}:
\begin{equation}
\nonumber
\bR(\Wv,\Ov) = \lambda\left(\frac{1}{\sigma^2}\|\Wv\|^2+\trace{\Wv\Ov^{-1}\Wv^T}~+d ~\log|\Ov|\right),~\lambda>0 \, .
\end{equation}
Unfortunately, such a regularizer is non-convex with respect to $\Ov$ due to the concavity of $\log|\Ov|$. To obtain a jointly convex formulation in $\Ov$ and $\Wv$, the authors in \cite{Zhang:2010ac} propose omitting $\log|\Ov|$ and controlling the complexity of $\Ov$ by adding a constraint on $\trace{\Ov}$:
\begin{equation}
\label{eq:prob}
\bR(\Wv,\Ov) = \lambda\left(\frac{1}{\sigma^2}\|\Wv\|^2+\trace{\Wv\Ov^{-1}\Wv^T}\right), ~\Ov\in\mathcal{Q} = \bigg\{\mathbf{Q}~|~\mathbf{Q}\succeq \0, ~\trace{\mathbf{Q}} = 1\bigg\} \, .
\end{equation}
It is worth noting that unlike the clustered MTL formulations, such as \eqref{eq:fixed_cluster}, the probabilistic formulation in \eqref{eq:prob} can model both positive and negative relationships among the tasks through the covariance matrix.

\paragraph{MTL with graphical models.}
Another way of modeling task relationships is through the precision matrix. This is popular in graphical models literature \cite{lauritzen1996graphical} because it encodes conditional independence among variables. In other words, if we denote the precision matrix among tasks in matrix variate Gaussian prior with $\Ov$, then $\Ov_{i,j} = 0$ if and only if tasks weights $\wv_i$ and $\wv_j$ are independent given the rest of the task weights \cite{gonccalves2016MTLcopula}. Therefore, assuming sparsity in the structure among the tasks translates to sparsity in matrix $\Ov$. As a result, we can formulate a sparsity-promoting regularizer by:
\begin{equation}
\label{eq:graph1}
\bR(\Wv,\Ov) = \lambda\left(\frac{1}{\sigma^2}\|\Wv\|^2+\trace{\Wv\Ov\Wv^T} -d~ \log|\Ov| \right)+ \lambda_1\|\Wv\|_1+\lambda_2\|\Ov\|_1 \, ,
\end{equation}
where $\lambda_1, \lambda_2\geq 0 $ control the sparsity of $\Wv$ and $\Ov$ respectively \cite{gonccalves2016MTLcopula}. It is worth noting that although this problem is jointly non-convex in $\Wv$ and $\Ov$, it is bi-convex.

\subsection{Strong Convexity of MTL Regularizers}
\label{app:M}
Recall that in Assumption \ref{asm:strong}, we presumed that the vectorized formulation of the MTL regularizer is strongly convex with respect to a matrix $\Mv^{-1}$. In this subsection we discuss the choice of matrix $\Mv$ for the widely-used MTL formulations introduced in Section~\ref{sec:mtlformulations}. 

Using the notation from Remark \ref{rmk:vecrep} for the clustered MTL formulation \eqref{eq:fixed_cluster_app}, it is easy to see that $\bar{\bR}(\wv,\bar{\Ov}) = \lambda_1\wv^T\bar{\Ov}\wv+\lambda_2\|\wv\|_2^2$, where $\bar{\Ov} \eqdef \Ov \otimes \mathbf{I}_{d\times d}$. As a result, it is clear that $\bar{\bR}(\wv,\bar{\Ov})$ is 1-strongly convex with respect to $\Mv^{-1} = \lambda_1\bar{\Ov}+\lambda_2\Iv_{md\times md}$.

Using a similar reasoning, it is easy to see that the matrix $\Mv$ can be chosen as $\lambda^{-1}(\eta\Iv+\bar{\Ov})$, ~$\lambda^{-1}(\frac{1}{\sigma^2}\Iv+\bar{\Ov}^{-1})^{-1}$~ and ~$\lambda^{-1}(\frac{1}{\sigma^2}\Iv+\bar{\Ov})^{-1}$ for \eqref{eq:cluster}, \eqref{eq:prob} and \eqref{eq:graph1} respectively.

\subsection{Optimizing $\Ov$ in MTL Formulations}
\label{sec:appendixomega}
In this section, we briefly cover approaches to update $\Ov$ in the MTL formulations introduced in Section~\ref{sec:mtlformulations}. First, it is clear that \eqref{eq:fixed_cluster} does not require any updates to $\Ov$, as it is assumed to be fixed. In \eqref{eq:cluster}, it can be shown \cite{zhou2011clustered, jacob2009clustered} that the optimal solution for $\Ov$ has the same column space as the rows of $\Wv$. Therefore, the problem boils down to solving a simple convex optimization problem over the eigenvalues of $\Ov$; see \cite{zhou2011clustered, jacob2009clustered} for details. Although outside the scope of this paper, we note that the bottleneck of this approach to finding $\Ov$ is computing the SVD of $\Wv$, which can be a challenging problem when $m$ is large. In the probabilistic model of \eqref{eq:prob}, the $\Ov$ update is given in \cite{Zhang:2010ac} by $(\Wv^T\Wv)^{\frac{1}{2}}$, which requires computing the eigenvalue decomposition of $\Wv^T\Wv$.
For the graphical model formulation, the problem of solving for $\Ov$ is called sparse precision matrix estimation or graphical lasso \cite{gonccalves2016MTLcopula}. This is a well-studied problem, and many scalable algorithms have been proposed to solve it \cite{wang2013LargeSparseCov,gonccalves2016MTLcopula,Hsieh:ub}.
\subsubsection{Reducing the Size of $\Ov$ by Sharing Tasks}
One interesting aspect of \cocoamtl is that the method can be easily modified to accommodate the sharing of tasks among the nodes without any change to the local solvers. This property helps the central node to reduce the size of $\Ov$ and the complexity of its update with minimal changes to the whole system. The following remark highlights this capability.
\begin{remark} \cocoamtl can be modified to solve problems when there are tasks that are shared among nodes. In this case, each node still solves a data local sub-problem based on its own data for the task, but the central node needs to do an additional aggregation step to add the results for all the nodes that share the data of each task. This reduces the size of matrix $\Ov$ and simplifies its update. 
\end{remark}

\section{Convergence Analysis}
\label{Apdx:ConvergenceAnalysis}

\paragraph{Notation.} In the rest of this section we use the superscript $(h)$ or $h$ to denote the corresponding variable at iteration $(h)$ of the federated update in \cocoamtl. When context allows, we drop the superscript to simplify notation.  

In order to provide a general convergence analysis, similar to the ones provided in \cite{Jaggi:2014cd, Ma:2015ti,Smith16}, we assume an aggregation parameter $\aggpar\in(0,1]$ in this section. With such an aggregation parameter, the updates in each federated iteration would be $\alphav_t \leftarrow \alphav_t + \aggpar\Delta\alphav_t$ and $\vv_t \leftarrow \vv_t + \aggpar\Delta\vv_t$. For a more detailed discussion on the role of aggregation parameter, see Appendix \ref{app:aggparRole}. Note that in Algorithm \ref{alg:mocha}, \cocoamtl is presented assuming $\aggpar=1$ for simplicity.

Before proving our convergence guarantees, we provide several useful definitions and  key lemmas.
\begin{definition}
\label{def:definitionOfSigmaK}
For each task $t$, define
\begin{equation}
\label{eq:definitionOfSigmaK}
\sigma_t \eqdef \max_{\alphav \in \R^{n_t}} \frac{\| \X_t \alphav \|_{M_t}^2}{\| \alphav \|^2} \textrm{\ \ and \ } \sigma_{\max} \eqdef \max_{t \in [m]} \sigma_t.
\end{equation}
\end{definition}

\begin{definition}
For any $\alphav$, define the duality gap as
\begin{equation}
\label{eq:gap}
G(\alphav) := \bD(\alphav)-(-\bP( \Wv(\alphav))),
\end{equation}
where $\bP(\Wv) := \sum_{t=1}^m \sum_{i=1}^{n_t} \ell_t(\wv_t^T \xv_t^i, y_t^i) +  \bR(\Wv, \Ov) $ as in~\eqref{eq:obj}.

\end{definition}

The following lemma uses Assumption \ref{asm:prob_stragglers} to bound the average performance of $\theta_t^h$, which is crucial in providing global convergence guarantees for \cocoamtl. 

\begin{lemma}
\label{lem:ThetaBound}
Under Assumption \ref{asm:prob_stragglers}, $\Theta_t^h\leq \bar{\Theta} = p_{\max} + (1-p_{\max})\Theta_{\max}<1$.
\end{lemma}

\begin{proof}
Recalling the definitions $p^h_{t} \eqdef \Prob[\theta^h_t = 1]$ and $\hat{\Theta}_t^h \eqdef \E [ \theta^h_t | \theta^h_t < 1, \bH_h ]$, we have
\begin{align}
  \Theta^h_t
  & =
  \E [ \theta^h_t | \bH_h ]  \nonumber\\
  & =
  \Prob[\theta^h_t = 1] \cdot \E [ \theta^h_t | \theta^h_t = 1, \bH_h ] + (1-\Prob[\theta^h_t < 1]) \cdot \E [ \theta^h_t | \theta^h_t < 1, \bH_h ]  \nonumber\\
  & =
  p^h_{t} \cdot 1 + (1-p^h_{t}) \cdot \hat{\Theta}_t^h \leq \bar{\Theta}< 1,\nonumber
\end{align}
where the last inequality is due to Assumption~\ref{asm:prob_stragglers}, and the fact that $\hat{\Theta}_t^h < 1$ by definition.
\end{proof}

The next key lemma bounds the dual objective of an iterate based on the dual objective of the previous iterate and the objectives of local subproblems.
\begin{lemma}
\label{lem:RelationOfDTOSubproblems}

For any
$\alphav, \Delta \alphav 
\in \R^{n}$ and $\aggpar\in(0,1]$ if $\sigma'$ satisfies~\eqref{eq:sigmaPrimeSafeDefinition}, then
\begin{equation}
  \bD\left(
\alphav +\aggpar \Delta \alphav
\right)
 \leq
 (1-\aggpar) \bD(\alphav)  + \aggpar
 \sum_{t=1}^m
 \Ggk(\Delta \alphav_t; \vv, \alphav_t) \, .
\end{equation}
\end{lemma}
\begin{proof}
The proof of this lemma is similar to \cite[Lemma 1]{Smith16} and follows from the definition of local sub-problems, smoothness of $\bR^{*}$ and the choice of $\sigma'$ in \eqref{eq:sigmaPrimeSafeDefinition}.
\end{proof}

Recall that if the functions $\ell_t$ are $(1/\mu)$-smooth, their conjugates $\ell^*_t$ will be $\mu$-strongly convex. The lemma below provides a bound on the amount of improvement in dual objective in each iteration.

\begin{lemma} 
\label{lem:basic-k}
If the functions $\ell_t^*$ are $\mu$-strongly convex for some $\mu \geq 0$. 
Then, for any $s\in [0,1]$.

\begin{align}
\mathbb{E}[\bD(\vc{\alphav}{h}) - \bD(\vc{\alphav}{h+1})|\bH_h]
& \geq
\aggpar \sum_{t=1}^m (1-\bar{\Theta}) \left(s G_t(\alphav^{(h)}) - \frac{\sigma's^2}{2} J_t \right),
\label{eq:lemma:dualdecrease_vs_dualitygap2}
\end{align}
where
\begin{align}
\label{eq:defOfG_k}
G_t(\alphav) \eqdef \sum_{i=1}^{n_t}  \left[ \ell_t^*(-\alphav_t^i) + \ell_t(\wv_t(\alphav)^\top \xv_t^i, y_t^i) + \alphav_t^i \wv_t(\alphav)^\top \xv_t^i \right] \, ,
\end{align}
\begin{align}
\label{eq:defOfR2}
J_t \eqdef - \tfrac{  \mu (1-s)}{\sigma' s } \|(\uv_t-\alphav^{(h)}_t)\|^2 + \| \X_t(\uv_t-\alphav^{(h)}_t) \|_{M_t}^2 \, ,
\end{align}
for $\uv_t \in \R^{n_t}$ with
\begin{equation}
\label{eq:defintionOfUi2}
\uv_t^i 
\in \partial \ell_t(\wv_t(\alphav)^\top \xv_t^i, y_t^i) \, .
\end{equation}
\end{lemma}
\begin{proof}
Applying Lemma~\ref{lem:RelationOfDTOSubproblems} and recalling $\bD(\alphav) = \sum_{t=1}^m \Ggk({\bf 0}; \vv, \alphav_t)$, we can first bound the improvement for each task separately.
Following a similar approach as in the proof of \cite[Lemma 7]{Smith16} we can obtain the bound \eqref{eq:lemma:dualdecrease_vs_dualitygap2} which bounds the improvement from $\vc{\alphav}{h}$ to $\vc{\alphav}{h+1}$.
\end{proof}

The following lemma relates the improvement of the dual objective in one iteration to the duality gap for the smooth loss functions $\ell_t$.

\begin{lemma}
\label{lem:oneRoundImprove}
If the loss functions $\ell_t$ are $(1/\mu)$-smooth, then there exists a proper constants $s\in (0,1]$ , such that for any $\aggpar\in(0,1]$ at any iteration $h$

\begin{align}
\E \left[ \bD(\vc{\alphav}{h}) - \bD(\vc{\alphav}{h+1}) | \bH_h \right]
\geq
s \aggpar (1-\bar\Theta) G(\vc{\alphav}{h}),
\label{eq:dualdecrease_vs_dualitygap4}
\end{align}
where $G(\vc{\alphav}{h})$ is the duality gap of $\alphav^{(h)}$ which is defined in \eqref{eq:gap}. 
\end{lemma}
\begin{proof}
Recall the definition of $\sigma_{\max}$ in ~\eqref{eq:definitionOfSigmaK}.
Now, if we carefully choose $s = \mu/( \mu + \sigma_{\max} \sigma')$, it is easy to show that $J_t\leq 0$ in \eqref{eq:lemma:dualdecrease_vs_dualitygap2}; see \cite[Theorem 11]{Smith16} for details. The final result follows as a consequence of Lemma \ref{lem:basic-k}.
\end{proof}

Note that Lemma \ref{lem:basic-k} holds even if the functions are non-smooth, i.e. $\mu=0$. However, we cannot infer sufficient decrease of Lemma \ref{lem:oneRoundImprove} from Lemma \ref{lem:basic-k} when $\mu=0$. Therefore, we need additional tools when the losses are $L$-Liptchitz. The first is the following lemma, which bounds the $J$ term in \eqref{eq:lemma:dualdecrease_vs_dualitygap2}.

\begin{lemma}
\label{lem:sigmabound}
Assume that the loss functions $\ell_t$ are $L$-Lipschitz.
Denote $J \eqdef \sum_{t=1}^m J_t$, where $J_t$ is defined in \eqref{eq:defOfR2}, then
\begin{align}
\label{eq:sigma}
J \leq 4L^2 \sum_{t=1}^m \sigma_t n_t \eqdef 4L^2 \sigma,
\end{align}
where $\sigma_t$ is defined in~\eqref{eq:definitionOfSigmaK}.
\end{lemma}
\begin{proof}
The proof is similar to \cite[Lemma 6]{Ma:2015ti} and using the definitions of $\sigma$ and $\sigma_t$ and the fact that the losses are $L$-Lipchitz.
\end{proof}

\subsection{Convergence Analysis for Smooth Losses}
\label{sec:appConvProof}
\subsubsection{Proof of Theorem~\ref{thm:convergenceSmoothCasePart1}}

Let us rewrite \eqref{eq:dualdecrease_vs_dualitygap4} from Lemma \ref{lem:oneRoundImprove} as
\begin{align}
\mathbb{E}[\bD(\alphav^{(h)})-\bD(\alphav^{(h+1)})|\bH_h] &= \bD(\alphav^{(h)}) - \bD(\alphav^{\star}) + \mathbb{E}[\bD(\alphav^{\star})-\bD(\alphav^{(h+1)})|\bH_h] \nonumber\\
&\geq s \aggpar (1-\bar\Theta) G(\vc{\alphav}{h})\nonumber\\
& \geq s \aggpar (1-\bar\Theta)\left( \bD(\alphav^{(h)})-\bD(\alphav^{\star}) \right),\nonumber
\end{align}
where the last inequality is due to weak duality, i.e. $G(\alphav^{(h)})\geq \bD(\alphav^{(h)})-\bD(\alphav^{\star})$. Re-arranging the terms in the above inequality, we can easily get
\begin{align}
\label{eq:finaldecrease}
\mathbb{E}[\bD(\alphav^{(h+1)})-\bD(\alphav^{\star})|\bH_h]\leq \left(1-s \aggpar (1-\bar\Theta)\right) \left(\bD(\alphav^{(h)})-\bD(\alphav^{\star})\right)
\end{align}
Recursively applying this inequality and taking expectations from both sides, we arrive at
\begin{align}
\mathbb{E}[\bD(\alphav^{(h+1)})-\bD(\alphav^{\star})]\leq \left(1-s \aggpar (1-\bar\Theta)\right)^{h+1}\left(\bD(\alphav^{(0)})-\bD(\alphav^{\star})\right).
\end{align}
Now we can use a simple bound on the initial duality gap \cite[Lemma 10]{Smith16}, which states that $\bD(\alphav^{(0)})-\bD(\alphav^{\star})\leq n$, to get the final result.
It is worth noting that we can translate the bound on the dual distance to optimality to the bound on the duality gap using the following inequalities
\begin{align}
s\aggpar(1-\bar{\Theta})~\Exp[G(\alphav^{(H)})]\leq\Exp[\bD(\vc{\alphav}{H})-\bD(\alphav^{(H+1)})]\leq\Exp[\bD(\vc{\alphav}{H})-\bD(\alphav^{\star})]\leq\epsilon_{\bD},
\end{align}
where the first inequality is due to \eqref{eq:dualdecrease_vs_dualitygap4}, the second inequality is due to the optimality of $\alphav^{\star}$, and the last inequality is the bound we just proved for the dual distance to optimality.

\subsubsection{Asymptotic Convergence}
In the case of smooth loss functions, it is possible to get asymptotic convergence results under milder assumptions. The following corollary is an extension of Theorem \ref{thm:convergenceSmoothCasePart1}.
\begin{corollary}
\label{coro:asym}
If the loss functions $\ell_t$ are $\mu$-smooth, then under Assumption \ref{asm:strong}, $\E [\bD(\vc{\alphav}{H})-\bD(\alphav^{\star}) ]\rightarrow 0$ as $H\rightarrow \infty$ if either of the following conditions hold
\begin{itemize}
\item $\limsup_{h\rightarrow\infty} p_t^h <1$ and $\limsup_{h\rightarrow\infty} \hat{\Theta}_t^h<1$.
\item For any task $t$, $\left(1- p_t^h\right)\times\left(1- \hat{\Theta}_t^h\right)=\omega(\frac{1}{h})$. Note that in this case $\lim_{h\rightarrow\infty} p_t^h$ can be equal to 1.
\end{itemize}
\end{corollary}
\begin{proof}
The proof is similar to the proof of Theorem \ref{thm:convergenceSmoothCasePart1}. We can use the same steps to get a sufficient decrease inequality like the one in \eqref{eq:finaldecrease}, with $\bar{\Theta}$ replaced with $\bar{\Theta}^{h} \eqdef \max_t~\Theta_t^h$.
\begin{align}
\mathbb{E}[\bD(\alphav^{(h+1)})-\bD(\alphav^{\star})|\bH_h]\leq \left(1-s \aggpar (1-\bar\Theta^h)\right) \left(\bD(\alphav^{(h)})-\bD(\alphav^{\star})\right)\nonumber
\end{align}
The rest of the argument follows by applying this inequality recursively and using the assumptions in the corollary.
\end{proof}

\subsection{Convergence Analysis for Lipschitz Losses: Proof for Theorem~\ref{thm:LipschitzLosses}}
\begin{proof}
For $L$-Lipschitz loss functions, the proof follows the same line of reasoning as the proof of Theorem 8 in \cite{Ma:2015ti} and therefore we do not cover it in detail. Unlike the case with smooth losses, it is not possible to bound the decrease in dual objective by \eqref{eq:dualdecrease_vs_dualitygap4}. However, we can use Lemma \ref{lem:basic-k} with $\mu=0$. The next step is to bound $J=\sum_{t=1}^m J_t$ in \eqref{eq:lemma:dualdecrease_vs_dualitygap2}, which can be done via Lemma \ref{lem:sigmabound}. Finally, we apply the inequalities recursively, choose $s$ carefully, and bound the terms in the final inequality. We refer the reader to the proof of Theorem 8 in \cite{Ma:2015ti} for more details. It is worth noting that similar to Theorem \ref{thm:convergenceSmoothCasePart1}, we can similarly get bounds on the expected duality gap, instead of the dual objective.
\end{proof}

\section{Choosing $\sigma'$}
In order to guarantee the convergence of the federated update of \cocoamtl, the parameter $\sigma'$ must satisfy:
\begin{equation}
\label{eq:sigmaPrimeSafeDefinition}  
\sigma' \sum_{t=1}^m \|\Xv_t\alphav_t\|_{\Mv_t}^2 \geq \aggpar\|\Xv\alphav\|_{\Mv}^2 \, \, ~\forall\alphav\in\mathbb{R}^n \,,
\end{equation}

where $\aggpar\in(0,1]$ is the aggregation parameter for \cocoamtl Algorithm. Note that in Algorithm \ref{alg:mocha} we have assumed that $\aggpar=1$. Based on Remark \ref{rmk:vecrep}, it can be seen that the matrix $\Mv$ in Assumption \ref{asm:strong} can be chosen of the form $\Mv = \bar{\Mv}\otimes \Iv_{d\times d}$, where $\bar{\Mv}$ is a positive definite matrix of size $m\times m$. For such a matrix, the following lemma shows how to choose $\sigma'$.

\label{sec:appsigma}
\begin{lemma}
\label{lemma:sigma}
For any positive definite matrix $\Mv = \bar{\Mv}\otimes \Iv_{d\times d}$,
\begin{equation}
\sigma' \eqdef \aggpar\max_{t}\sum_{t'=1}^m \frac{|\bar{\Mv}_{tt'}|}{\bar{\Mv}_{tt}}
\end{equation}
satisfies the inequality \eqref{eq:sigmaPrimeSafeDefinition}.
\end{lemma}
\begin{proof}
First of all it is worth noting that for any $t$, $\Mv_t = \bar{\Mv}_t\otimes \Iv_{d\times d}$. For any $\alphav\in\mathbb{R}^n$
\begin{align}
\aggpar\|\Xv\alphav\|_{\Mv}^2 &= \gamma\sum_{t,t'} \bar{\Mv}_{tt'}\langle \Xv_t\alphav_t,\Xv_{t'}\alphav_{t'}\rangle\nonumber\\
&\leq \aggpar\sum_{t,t'} \frac{1}{2}|\bar{\Mv}_{tt'}|\left(\frac{1}{\bar{\Mv}_{tt}}\|\Xv_t\alphav_t\|_{\Mv_t}^2 + \frac{1}{\bar{\Mv}_{t't'}}\|\Xv_{t'}\alphav_{t'}\|_{\Mv_{t'}}^2\right)\nonumber\\
&= \aggpar\sum_t \left(\sum_{t'}\frac{|\bar{\Mv}_{tt'}|}{\bar{\Mv}_{tt}}\right)\|\Xv_t\alphav_t\|_{\Mv_t}^2\nonumber\\
&\leq \sigma' \sum_{t} \|\Xv_t\alphav_t\|_{\Mv_t}^2,\nonumber
\end{align}
where the first inequality is due to Cauchy-Schwartz and the second inequality is due to definition of $\sigma'$.
\end{proof}
\begin{remark}
Based on the proof of Lemma \ref{lemma:sigma}, it is easy to see that we can choose $\sigma'$ differently across the tasks in our algorithm to allow tasks that are more loosely correlated with other tasks to update more aggressively. To be more specific, if we choose $\sigma'_t = \aggpar\sum_{t'}\frac{|\bar{\Mv}_{tt'}|}{\bar{\Mv}_{tt}}$, then it it is possible to show that $\aggpar\|\Xv\alphav\|_{\Mv}^2\leq \sum_{t=1}^{m}\sigma'_t\|\Xv_t\alphav_t\|_{\Mv_t}^2$ for any $\alphav$, and the rest of the convergence proofs will follow.
\end{remark}
\subsection{The Role of Aggregation Parameter $\aggpar$}
\label{app:aggparRole}
The following remark highlights the role of aggregation parameter $\aggpar$.
\begin{remark}
Note that the when $\aggpar<1$ the chosen $\sigma'$ in \eqref{eq:sigmaPrimeSafeDefinition} would be smaller compared to the case where $\aggpar=1$. This means that the local subproblems would be solved with less restrictive regularizer. Therefore, the resulting $\Delta\alphav$ would be more aggressive. As a result, we need to do a more conservative update $\alphav+\aggpar\Delta\alphav$ in order to guarantee the convergence.
\end{remark}
Although aggregation parameter $\aggpar$ is proposed to capture this trade off between aggressive subproblems and conservative updates, in most practical scenarios $\aggpar=1$ has the best empirical performance.

\section{Simulation Details}
\label{sec:expdetails}

In this section, we provide additional details and results of our empirical study.

\subsection{Datasets}

In Table~\ref{table:data}, we provide details on the number of tasks ($m$), feature size ($d$), and per-task data size ($n_t$) for each federated dataset described in Section~\ref{sec:experiments}. The standard deviation $n_{\sigma}$ is a measure data skew, and calculates the deviation in the sizes of training data points for each task, $n_t$. All datasets are publicly available.
\begin{table}[h]
\centering
\small
\caption{Federated Datasets for Empirical Study.}
\vspace{1em}
\begin{tabular}{c c c c c c}
\toprule
{\textbf{Dataset}} & \textbf{Tasks ($m$)} & \textbf{Features ($d$)} & \textbf{Min $n_t$} & \textbf{Max $n_t$} & \textbf{Std. Deviation $n_{\sigma}$}  \\ \toprule
Human Activity & 30 & 561 & 210 & 306 & 26.75 \\ \midrule
Google Glass & 38 & 180 & 524 &  581 & 11.07 \\ \midrule

Vehicle Sensor & 23 & 100 & 872 & 1,933 & 267.47 \\
\bottomrule
\end{tabular}
\label{table:data}
\end{table}

\subsection{Multi-Task Learning with Highly Skewed Data}

\begin{table}[h]
\centering
\small
\caption{Skewed Datasets for Empirical Study.}
\vspace{1em}
\begin{tabular}{c c c c c c}
\toprule
{\textbf{Dataset}} & \textbf{Tasks ($m$)} & \textbf{Features ($d$)} & \textbf{Min $n_t$} & \textbf{Max $n_t$} & \textbf{Std. Deviation $\sigma$}  \\ \toprule
HA-Skew & 30 & 561 & 3 & 306 & 84.41 \\ \midrule
GG-Skew & 38 & 180 & 6 & 581 & 111.79 \\ \midrule

VS-Skew & 23 & 100 & 19  & 1,933 & 486.08 \\
\bottomrule
\end{tabular}
\label{table:dataskew}
\end{table}

To generate highly skewed data, we sample from the original training datasets so that the task dataset sizes differ by at least two orders of magnitude. The sizes of these highly skewed datasets are shown in Table~\ref{table:dataskew}. When looking at the performance of local, global, and multi-task models for these datasets (Table~\ref{table:skew}), the global model performs slightly better in this setting (particularly for the Human Activity dataset). However, multi-task learning still significantly outperforms all models.

\begin{table}[h]
\centering
\caption{Average prediction error for skewed data: means and standard errors over 10 random shuffles.}
\vspace{1em}
\begin{tabular}{c c c c  c}
\toprule
{\textbf{Model}} & HA-Skew & GG-Skew  & VS-Skew  \\ \toprule
Global & 2.41 (0.30) & 5.38 (0.26) &  13.58 (0.23) \\
\midrule
Local & 3.87 (0.37) & 4.96 (0.20)  &  8.15 (0.19)  \\
\midrule
MTL & \textbf{1.93 (0.44)} & \textbf{3.28 (0.15)}  & \textbf{6.91 (0.21)}  \\
\bottomrule
\end{tabular}
\label{table:skew}
\end{table}

\subsection{Implementation Details}

In this section, we provide details on our experimental setup and compared methods.

\paragraph{Methods.}
\begin{itemize}[leftmargin=*]
\item\textbf{Mb-SGD.} Mini-batch stochastic gradient descent is a standard, widely used method for parallel and distributed optimization. See, e.g., a discussion of this method for the SVM models of interest~\cite{ShalevShwartz:2007p2580}. We tune both the mini-batch size and step size for best performance using
grid search.
\item\textbf{Mb-SDCA.} Mini-batch SDCA aims to improve mini-batch SGD by employing coordinate ascent in the dual,
which has encouraging theoretical and practical backings~\cite{ShalevShwartz:2013wl,Takac:2013ut}. For all experiments, we scale the updates for mini-batch stochastic dual coordinate ascent at each round by $\frac{\beta}{b}$ for mini-batch size $b$ and $\beta \in [1, b]$, and tune both parameters with grid search.
\item\textbf{\cocoa.} We generalize \cocoa~\cite{Jaggi:2014cd, Ma:2015ti} to solve~\eqref{eq:obj}, and tune $\theta$, the fixed approximation parameter, between $[0,1)$ via grid search. For both \cocoa, and \cocoamtl, we use coordinate ascent as a local solver for the dual subproblems~\eqref{eq:subproblem}.
\item\textbf{\cocoamtl.} The only parameter necessary to tune for \cocoamtl is the level of approximation quality $\theta_t^h$, which can be directly tuned via $H_i$, the number of local iterations of the iterative method run locally. In Section~\ref{sec:convergence}, our theory relates this parameter to global convergence, and we discuss the practical effects of this parameter in Section~\ref{ssec:syschal}.
\end{itemize}

\paragraph{Computation and Communication Complexities.} We provide a brief summary of the above methods from the point of view of computation, communication, and memory complexities. \cocoamtl is superior in terms of its computation complexity compared to other distributed optimization methods, as \cocoamtl allows for flexibility in its update of W. At one extreme, the update can be based on a single data point per iteration in parallel, similar to parallel SGD. At the other extreme, \cocoamtl can completely solve the subproblems on each machine, similar to methods such as ADMM. This flexibility of computation yields direct benefits in terms of communication complexity, as performing additional local computation will result in fewer communication steps. Note that all methods, including \cocoamtl, communicate the same size vector at each iteration, and so the main difference is in how many communication rounds are necessary for convergence. In terms of memory, \cocoamtl must maintain the task matrix, $\Omega$, on the master server. While this overhead is greater than most \textit{non-MTL} (global or local) approaches, the task matrix is typically low-rank by design and the overhead is thus manageable. We discuss methods for computing $\Omega$ in further detail in Section~\ref{sec:appendixomega}.

\paragraph{Estimated Time.}
To estimate the time to run methods in the federated setting, we carefully count the floating-point operations (FLOPs) performed in each local iteration for each method, as well as the size and frequency of communication. We convert these counts to estimated time (in milliseconds), using known clock rate and bandwidth/latency numbers for mobile phones in 3G, LTE, and wireless networks~\cite{van2009multi, Huang2013an, singelee2011communication, Carroll2010aa, Miettinen2010ee}. In particular, we use the following standard model for the cost of one round, $h$, of local computation / communication on a node $t$:
\begin{equation}
\textit{Time}(h,t) := \frac{\textit{FLOPs}(h,t)}{\textit{Clock Rate}(t)} + \textit{Comm}(h, t)
\end{equation}
Note that the communication cost $\textit{Comm}(h,t)$ includes both bandwidth and latency measures. Detailed models of this type have been used to closely match the performance of real-world systems~\cite{qipaleo}.

\paragraph{Statistical Heterogeneity.} To account for statistical heterogeneity, \cocoamtl and the mini-batch methods (Mb-SGD and Mb-SDCA) can adjust the number of local iterations or batch size, respectively, to account for difficult local problems or high data skew. However, because \cocoa uses a fixed accuracy parameter $\theta$ across both the tasks and rounds, changes in the subproblem difficulty and data skew can make the computation on some nodes much slower than on others. For \cocoa, we compute $\theta$ via the duality gap, and carefully tune this parameter between $[0,1)$ for best performance. Despite this, the number of local iterations needed for $\theta$ varies significantly across nodes, and as the method runs, the iterations tend to increase as the subproblems become more difficult. 

\paragraph{Systems Heterogeneity.} Beyond statistical heterogeneity, there can be variability in the systems themselves that cause changes in performance. For example, low battery levels, poor network connections, or low memory may reduce the ability a solver has on a local node to compute updates.
As discussed in Section~\ref{ssec:syschal}, \cocoamtl assumes that the central node sets some global clock cycle, and the $t$-th worker determines the amount of feasible local computation given this clock cycle along with its systems constraints. This specified amount of local computation corresponds to some implicit value of  $\theta_t^h$ based on the underlying systems and statistical challenges for the $t$-th node.

To model this setup in our simulations, it suffices to fix a global clock sycle and then randomly assign various amounts of local computation to each local node at each iteration. Specifically, in our simulations we charge all nodes the same fixed computation cost at each iteration over an LTE network, but force some nodes to perform less updates given their current systems constraints. At each round, we assign the number of updates for node $t$ between $[0.1n_{\min}, n_{\min}]$ for \textit{high variability} environments, and between $[0.9n_{\min}, n_{\min}]$ for \textit{low variability} environments, where $n_{\min} := \min_t n_t$ is the minimum number of local data points across tasks.  

For the mini-batch methods, we vary the mini-batch size in a similar fashion. However, we do not follow this same process for \cocoa, as this would require making the $\theta$ parameter worse than what was optimally tuned given statistical heterogeneity. Hence, in these simulations we do not introduce any additional variability (and thus present overly optimistic results for \cocoa).  In spite of this, we see that in both low and high variability settings, \cocoamtl significantly outperforms all other methods and is robust to systems-related heterogeneity.

\paragraph{Fault Tolerance.} Finally, we demonstrate that \cocoamtl can handle nodes periodically dropping out, which is also supported in our convergence results in Section~\ref{sec:convergence}. We perform this simulation using the notation defined in Assumption~\ref{asm:prob_stragglers}, i.e., that each node $t$ temporarily drops on iteration $h$ with probability $p_t^h$. In our simulations, we modify this probability directly and show that \cocoamtl is robust to fault tolerance in Figure~\ref{fig:ft}. However, note that this robustness is not merely due to statistical redundancy: If we are to drop out a node entirely (as shown in the green dotted line), \cocoamtl will not converge to the correct solution. This provides insight into our Assumption~\ref{asm:prob_stragglers}, which requires that the probability that a node drops at each round cannot be exactly equal to one.

\end{document}